\documentclass[twoside,11pt]{article}
\usepackage{jmlr2e}
\usepackage{times}

\usepackage{amsmath,amssymb}
\usepackage{booktabs}
\usepackage{colortbl}
\usepackage{array}
\usepackage{enumitem}
\usepackage{mdframed}
\usepackage{multirow}
\usepackage{pgfplots}
\pgfplotsset{compat=1.18}
\usepackage{tikz}
\usetikzlibrary{arrows.meta, positioning, shapes.geometric,
                fit, backgrounds, decorations.pathreplacing}
\usepackage{xcolor}
\usepackage{subcaption}

\definecolor{colB0}{RGB}{136,135,128}       
\definecolor{colVAT}{RGB}{216,90,48}        
\definecolor{colPGD2}{RGB}{212,83,126}      
\definecolor{colPGD4}{RGB}{180,50,100}      
\definecolor{colNoPMH}{RGB}{99,140,28}      
\definecolor{colPMH}{RGB}{55,138,221}       
\definecolor{colPMHlight}{RGB}{181,212,244} 
\definecolor{colRed}{RGB}{226,75,74}        
\definecolor{colGreen}{RGB}{15,110,86}      

\pgfplotsset{
  /pgfplots/pmh legend below/.style={
    legend style={
      font=\scriptsize,
      draw=none,
      fill=none,
      at={(0.5,-0.42)},
      anchor=north,
      legend columns=-1,
      /tikz/every even column/.append style={column sep=1.1em},
      cells={anchor=west},
      inner xsep=1pt,
    },
    legend image code/.code={
      \draw[#1] (0cm,-0.08cm) rectangle (0.28cm,0.12cm);
    },
  },
  /pgfplots/pmh legend north/.style={
    legend style={
      font=\scriptsize, draw=none, fill=none,
      at={(0.5,1.05)}, anchor=south, legend columns=-1,
      /tikz/every even column/.append style={column sep=0.7em},
    },
  },
  /pgfplots/pmh legend nw/.style={
    legend style={
      font=\scriptsize, draw=none, fill=white, fill opacity=0.92,
      text opacity=1,
      at={(0.03,0.97)}, anchor=north west,
      inner sep=2pt,
    },
  },
  pmh base/.style={
    width=\linewidth,
    height=4.6cm,
    tick label style={font=\scriptsize},
    label style={font=\small},
    title style={font=\small\bfseries},
    ylabel style={font=\small},
    xlabel style={font=\small},
    grid=major,
    grid style={line width=0.3pt, draw=gray!22},
    minor grid style={draw=none},
    line width=1.0pt,
    mark size=1.8pt,
    enlarge x limits=false,
    clip=false,
    pmh legend below,
  },
  pmh base tall/.style={
    pmh base,
    height=5.4cm,
  },
  pmh bar/.style={
    pmh base,
    ybar=1.2pt,
    bar width=5.5pt,
    enlarge x limits=0.16,
    xtick align=inside,
    ymin=0,
    pmh legend below,
    legend image code/.code={
      \fill[#1] (0cm,-0.08cm) rectangle (0.28cm,0.12cm);
    },
  },
  pmh hbar/.style={
    pmh base,
    xbar=1pt,
    bar width=7pt,
    enlarge y limits=0.16,
    xmin=0,
  },
  pmh series erm/.style={color=colB0, mark=*, solid},
  pmh series vat/.style={color=colVAT, mark=square*, solid},
  pmh series pgd2/.style={color=colPGD2, mark=triangle*, dashed},
  pmh series pgd4/.style={color=colPGD4, mark=triangle*, solid},
  pmh series twoview/.style={color=colNoPMH, mark=diamond*, dashed},
  pmh series pmh/.style={color=colPMH, mark=*, solid, line width=1.35pt},
  pmh fill erm/.style={fill=colB0, draw=none},
  pmh fill vat/.style={fill=colVAT, draw=none},
  pmh fill twoview/.style={fill=colNoPMH, draw=none},
  pmh fill pmh/.style={fill=colPMH, draw=none},
  pmh fill pgd4/.style={fill=colPGD4, draw=none},
  pmh fill pgd2/.style={fill=colPGD2, draw=none},
}

\tikzset{
  pmharrow/.tip={Latex[length=4pt,width=3pt]},
  pmh dot/.style={circle, fill, inner sep=1.1pt},
  pmh panel/.style={draw=gray!40, rounded corners=2pt, minimum width=3.4cm,
                    minimum height=3.5cm},
  pmh badge/.style={font=\tiny\bfseries, rounded corners=1.5pt,
                    inner sep=2pt, align=center},
}

\graphicspath{{images/}}

\mdfdefinestyle{plainlang}{%
  leftmargin=4pt, rightmargin=4pt,
  innerleftmargin=6pt, innerrightmargin=6pt,
  innertopmargin=4pt, innerbottommargin=4pt,
  linewidth=0.5pt, linecolor=black!60,
  backgroundcolor=gray!7,
  skipabove=4pt, skipbelow=4pt}

\usepackage{lastpage}
\jmlrheading{0}{2026}{1-\pageref{LastPage}}{7/26}{}{}{Vishal Rajput}
\ShortHeadings{Supervised Learning Has a Geometric Blind Spot}{Rajput}
\firstpageno{1}

\begin{document}

\title{Supervised Learning Has a Geometric Blind Spot}

\author{\name Vishal Rajput \email vishal.stark42@gmail.com \\
       \addr Department of Computer Science \\
       KU Leuven \\
       Celestijnenlaan 200A, 3001 Leuven, Belgium}

\editor{}

\maketitle

\begin{abstract}%
Ordinary supervised training minimises the task loss and then stops.
It never pays for how far the representation moves when the input is nudged
along directions that helped fit training labels---including directions that
are nuisance at deployment.
We call that leftover sensitivity the geometric blind spot of empirical risk
minimisation.
In a Gaussian linear model where the nuisance enters the label conditional and
the decoder has finite Lipschitz constant, population MSE forces a floor on
linearised representation drift.
The same distinction predicts a failure mode of adversarial training: Jacobian
magnitude can fall while clean class geometry worsens.
We track that dissociation with a class-layout score and study isotropic
encoder matching---penalising $\|\phi(x)-\phi(x+\delta)\|^2$ for Gaussian
$\delta$ under a task-loss cap---when nuisance axes are unknown.
On a Vision Transformer trained from scratch on CIFAR-10, projected gradient
descent attains the smallest Jacobian Frobenius yet the worst clean layout
score ($1.353{\pm}0.020$ over three seeds), above task-only training ($1.093$);
isotropic matching attains the best ($0.904$).
The drift floor is proved for the linear-Gaussian case; deep nets and
cross-task orderings are protocol empirics.
Design rule: report class-layout geometry beside the task score; prefer
isotropic encoder matching when axes are unknown.
\end{abstract}


\section{Introduction}
\label{sec:intro}

Ordinary supervised training cares about one thing: the task loss.
Once that loss is small, training is done.
Nothing in that objective pays for how far the embedding moves when the input
is nudged along directions that predicted labels in the training data.
When those directions are nuisance at deployment, models can look
``broken'' under mild shifts not because deep nets cannot be robust, but
because the loss never asked otherwise.

A chest X-ray classifier trained where scanner artefacts correlate with
diagnosis is the motivating picture: residual sensitivity to scanner noise
need have nothing to do with pathology.
We treat such examples as hypotheses about training-time direct influence
unless that condition is established.

We call residual sensitivity to training-predictive nuisance the
\emph{geometric blind spot} of empirical risk minimisation (ERM).
In a Gaussian linear direct-influence model under population MSE, with
finite decoder Lipschitz constant~$L$, any population minimiser keeps a floor
on linearised embedding drift along the nuisance of order $\rho^2/L^2$
(Theorem~1).
Isotropic encoder matching---Predictive Matching Hygiene (PMH)---is one
Proposition~5-motivated member of the stability/Jacobian family.
The Trajectory Deviation Index (TDI) in our tables is the intra-/inter-class
latent distance ratio under isotropic input noise; Task~04 method orderings
are the empirics we hold that story accountable for.

Keep three claims separate (expanded in \S\ref{sec:theory}):
when $n$ enters $p(y\mid x)$, the Bayes predictor depends on~$n$;
calling $n$ a deployment nuisance is an interpretive data story;
under a finite decoder Lipschitz factorisation, population MSE forces a
quantitative lower bound on linearised encoder drift---the nontrivial claim.

\paragraph{A picture.}
Projected gradient descent (PGD) suppresses sensitivity in a worst-case
direction; isotropic probes predict redistribution rather than removal
(Corollary~4).
Figure~\ref{fig:tdi-curves} shows the Task~04 ordering: PGD lowest Jacobian
Frobenius yet worse clean TDI than task-only training; PMH lowest TDI
(Table~\ref{tab:main}).
The figure illustrates that ordering; it is not a proof of Theorem~1.


\begin{figure}[t]
\centering
\begin{subfigure}[t]{\linewidth}
\centering
\begin{tikzpicture}
\begin{axis}[
  pmh base,
  width=0.92\linewidth,
  height=4.8cm,
  xlabel={Perturbation strength $\sigma$},
  ylabel={TDI (lower $=$ better)},
  xmin=-0.005, xmax=0.21,
  ymin=0.70, ymax=2.70,
  legend columns=6,
]
\addplot[pmh series erm]
  coordinates {(0,1.093)(0.05,1.215)(0.10,1.584)(0.15,2.053)(0.20,2.468)};
\addlegendentry{ERM}
\addplot[pmh series vat]
  coordinates {(0,1.276)(0.05,1.317)(0.10,1.524)(0.15,1.790)(0.20,2.037)};
\addlegendentry{VAT}
\addplot[pmh series pgd2]
  coordinates {(0,1.280)(0.05,1.310)(0.10,1.440)(0.15,1.620)(0.20,1.870)};
\addlegendentry{PGD-2/255}
\addplot[pmh series pgd4]
  coordinates {(0,1.353)(0.05,1.372)(0.10,1.461)(0.15,1.648)(0.20,1.892)};
\addlegendentry{PGD-4/255}
\addplot[pmh series twoview]
  coordinates {(0,1.074)(0.05,1.139)(0.10,1.262)(0.15,1.414)(0.20,1.646)};
\addlegendentry{two-view}
\addplot[pmh series pmh]
  coordinates {(0,0.904)(0.05,0.946)(0.10,1.050)(0.15,1.150)(0.20,1.305)};
\addlegendentry{PMH}
\end{axis}
\end{tikzpicture}
\caption{TDI under increasing noise.}
\label{fig:tdi-curves}
\end{subfigure}

\vspace{0.85em}

\begin{subfigure}[t]{\linewidth}
\centering
\begin{tikzpicture}
\begin{axis}[
  pmh base,
  width=0.92\linewidth,
  height=4.8cm,
  xlabel={Jacobian Frobenius norm (log scale)},
  ylabel={TDI at $\sigma{=}0$ (lower $=$ better)},
  xmode=log,
  xmin=1.8, xmax=70,
  ymin=0.82, ymax=1.48,
  legend columns=5,
]
\addplot[only marks, mark=*, mark size=3pt, color=colB0]
  coordinates {(34.58,1.093)};
\addlegendentry{ERM}
\addplot[only marks, mark=square*, mark size=3pt, color=colVAT]
  coordinates {(5.01,1.276)};
\addlegendentry{VAT}
\addplot[only marks, mark=triangle*, mark size=3.5pt, color=colPGD4]
  coordinates {(2.99,1.353)};
\addlegendentry{PGD-4/255}
\addplot[only marks, mark=diamond*, mark size=3pt, color=colNoPMH]
  coordinates {(13.09,1.074)};
\addlegendentry{two-view}
\addplot[only marks, mark=*, mark size=3.5pt, color=colPMH]
  coordinates {(8.08,0.904)};
\addlegendentry{PMH}
\draw[dashed, gray!45, line width=0.5pt]
  (axis cs:6,0.82) -- (axis cs:6,1.48);
\draw[dashed, gray!45, line width=0.5pt]
  (axis cs:1.8,1.15) -- (axis cs:70,1.15);
\node[font=\scriptsize, align=center, color=colRed!80!black, anchor=north]
  at (axis cs:3.2,1.42) {output patching};
\node[font=\scriptsize, align=center, color=colGreen, anchor=south]
  at (axis cs:28,0.88) {geometric repair};
\end{axis}
\end{tikzpicture}
\caption{Jac.\ Fro vs.\ clean-input TDI.}
\label{fig:jac-scatter}
\end{subfigure}

\caption{\textbf{Clean-input geometry on Task~04 (ViT/CIFAR-10 from scratch).}
\emph{(a)} TDI curves: PMH lowest at every plotted noise level.
PGD starts at TDI $1.353$ on clean inputs (mean over seeds 42--44), worse than
ERM (1.093), matching Corollary~4's prediction.
\emph{(b)} PGD lowest Jac.\ Fro ($2.99$) yet worst TDI ($1.353$); PMH lowest TDI
(0.904) with moderate Jacobian reduction.
Left of the vertical guide: lower Fro / higher TDI (output patching);
below the horizontal guide: better clean TDI.
PGD geometry: multi-seed means; other methods: seed-42 protocol.}
\label{fig:fig1}
\end{figure}

\paragraph{Four phenomena as readings.}
Adversarial vulnerability, texture bias, corruption fragility, and the
robustness--accuracy tradeoff have each received separate explanations.
Under direct correlated nuisance they can share a geometric contribution from
ERM---one factor among others (\S\ref{sec:readings}).

\paragraph{Contributions.}
\begin{itemize}[leftmargin=*,topsep=2pt,itemsep=1pt]
\item \textbf{Encoder-drift floor.}
  In the Gaussian linear direct-influence model under population MSE,
  $\tilde D(\phi^*,\sigma)\geq\sigma^2\rho^2 C(P)/L^2>0$
  (\S\ref{sec:theory}).
\item \textbf{Directional vs.\ isotropic control.}
  Corollary~4: PGD need not enforce isotropic Jacobian shrinkage.
  On Task~04, Fro falls while clean TDI rises; PMH attains the lowest TDI
  (\S\ref{sec:metrics}).
\item \textbf{Isotropic covariance for a chosen map.}
  Proposition~5 equates matching to $\|J\|_F$ control iff the perturbation
  covariance is isotropic; PMH applies that at selected encoder blocks under
  a task-loss cap (\S\ref{sec:metrics}).
\item \textbf{Related readings.}
  Four robustness programmes as geometric readings under direct influence
  (\S\ref{sec:readings}).
\item \textbf{Diagnostics across setups.}
  Cross-task geometry, BERT/SST-2 surface-form trends, and ImageNet ViT-B/16
  under stated protocols (\S\ref{sec:expts}).
\end{itemize}
Theorem~1 is proved for Gaussian linear MSE with direct influence and
finite~$L$; Corollary~2 is qualitative; deep-net blocks and screening-type
language settings are experiments and diagnostics, not completed nonlinear
proofs.

\paragraph{Glossary.}
{\small
\begin{itemize}[leftmargin=*,topsep=2pt,itemsep=1pt]
\item \textbf{ERM / task-only} (alias B0): empirical risk minimisation---train
  with the task loss alone.
\item \textbf{PGD / VAT:} projected gradient descent adversarial
  training~\citep{madry2018pgd}; virtual adversarial
  training~\citep{miyato2018virtual}.
\item \textbf{Geometric blind spot} (our term): residual encoder sensitivity
  that task loss alone need not remove under Theorem~1's assumptions;
  not a claim that ERM fails in general.
\item \textbf{TDI} (our diagnostic): Trajectory Deviation Index---intra-class
  mean pairwise latent distance divided by mean inter-centroid distance,
  optionally under isotropic input noise (lower $=$ tighter class layout).
\item \textbf{PMH} (alias E1): Predictive Matching Hygiene---our name for
  \emph{isotropic encoder matching}:
  $\|\phi(x)-\phi(x+\delta)\|^2$ with $\delta\sim\mathcal{N}(0,\sigma^2 I)$,
  under a task-loss cap. Closest literature neighbours are consistency /
  Jacobian regularisers (\S\ref{sec:readings}).
  \textbf{Two-view} (alias E1\_no\_pmh) uses the same two forward passes without
  the matching term.
\item \textbf{Paraphrase embedding drift:} mean $\|\Delta\texttt{[CLS]}\|$ of a
  BERT sentence embedding under synonym paraphrase (SST-2 $=$ Stanford Sentiment
  Treebank polarity probe).
  \textbf{Para~/~non-para ratio:} that drift divided by drift under matched
  non-paraphrase edits (relative surface-form sensitivity).
\item \textbf{Direct vs.\ screening:} direct means $I(n;y\mid s)>0$;
  screening ($I(n;y\mid s)=0$) does not trigger Theorem~1 / Cor.~2.
\item \textbf{Isotropy index}
  $\mathcal{A}:=\mathbb{E}\|J\|_F^2/\mathbb{E}\|Jw\|_2^2\geq 1$:
  equals $1$ at rank-1; up to $d_x$ when isotropic.
\end{itemize}
}

\section{Setup and Three Objectives}
\label{sec:objectives}
\label{sec:setup}

Write $x$ through measurable factors $s(x)$ (signal) and $n(x)$ (candidate
nuisance).
\emph{Direct influence} means $I(n;y\mid s)>0$: $n$ enters the training
conditional of~$y$.
\emph{Screening} means $I(n;y\mid s)=0$: $n$ may correlate with~$y$
marginally without changing the conditional given~$s$.
Theorem~1 and Corollary~2 apply only under direct influence.
Calling $n$ irrelevant at deployment is an interpretive data story, not part
of the proof.

We compare three training objectives (Figure~\ref{fig:fig2}).
Proposition~5 motivates isotropic covariance once the matched map is the
encoder.

\paragraph{ERM (task-only).}
\begin{equation}
\theta^*\in\arg\min_\theta\;\mathbb{E}_{(x,y)\sim P}\!\left[\mathcal{L}_{\text{task}}(f_\theta(x),y)\right].\tag{1}
\end{equation}
No geometric term.
Under Theorem~1's assumptions, linearised encoder drift along~$n$ cannot be
zeroed while remaining a population MSE minimiser with finite decoder
Lipschitz constant~$L$.

\paragraph{PGD (adversarial training).}
\begin{equation}
\min_\theta\;\max_{\|\delta\|_\infty\leq\varepsilon}\mathbb{E}\!\left[\mathcal{L}_{\text{task}}(f_\theta(x+\delta),y)\right].\tag{2}
\end{equation}
Suppresses the Jacobian along a worst-case direction.
Sensitivity can redistribute (low $\mathcal{A}$), so clean TDI can worsen even
when $\|J\|_F$ falls (Table~\ref{tab:main}).

\paragraph{PMH (isotropic encoder matching).}
\begin{align}
&\min_\theta\;\mathbb{E}\!\left[\mathcal{L}_{\text{task}}(f_\theta(x),y)+\lambda w(t)\,\mathcal{L}_{\text{PMH}}\right],\nonumber\\
&\mathcal{L}_{\text{PMH}}=\|\phi_\theta(x)-\phi_\theta(x+\delta)\|^2,\;\delta\sim\mathcal{N}(0,\sigma^2 I).\tag{3}
\end{align}
As $\sigma\to 0$, $\mathcal{L}_{\text{PMH}}\approx\sigma^2\|J_\phi\|_F^2$.
Isotropic Gaussian is the unique zero-mean covariance equating matching to
$\|J\|_F$ control for the matched map (Proposition~5); PMH applies that at
selected encoder blocks under a task-loss cap---one member of the
stability/Jacobian family.
$w(t)$ is a cosine warmup ramp.


\begin{figure}[t]
\centering
\resizebox{\linewidth}{!}{%
\begin{tikzpicture}[
  font=\small,
  dot/.style={circle, fill, inner sep=1.2pt},
  axarr/.style={-{Latex[length=3.5pt]}, gray!55, line width=0.5pt},
  panel/.style={draw=gray!45, rounded corners=3pt, inner sep=6pt,
                line width=0.5pt, minimum width=4.55cm, minimum height=4.35cm},
  badge/.style={rounded corners=2pt, inner sep=2.5pt, font=\scriptsize\bfseries,
                align=center, text width=4.1cm},
  paneltitle/.style={font=\small\bfseries, anchor=south},
  panelsub/.style={font=\scriptsize, gray!70!black, anchor=north, align=center,
                   text width=4.2cm},
]

\begin{scope}[xshift=0cm]
  \node[paneltitle] at (0,2.55) {ERM training};
  \node[panelsub] at (0,2.28) {Jacobian rough along nuisance};
  \node[panel, fill=gray!3] at (0,0) {};
  \draw[axarr] (-1.85,-1.55) -- (1.85,-1.55)
    node[right,font=\scriptsize]{signal $s$};
  \draw[axarr] (-1.85,-1.55) -- (-1.85,1.45)
    node[above,font=\scriptsize]{nuisance $n$};
  \foreach \x/\y in {-0.3/0.55,0.2/0.85,0.7/0.35,-0.6/0.15,0.9/0.75,
                      -0.1/-0.15,0.5/-0.35,-0.8/0.75,0.3/-0.05,1.0/-0.15,
                      -0.4/-0.55,0.8/-0.65}{
    \node[dot, colB0, opacity=0.75] at (\x,\y) {};
  }
  \node[badge, fill=colRed!12, text=colRed!75!black] at (0,-2.35)
    {Large TDI $\uparrow$ \quad TDI@0 $=1.093$};
\end{scope}

\begin{scope}[xshift=5.15cm]
  \node[paneltitle] at (0,2.55) {Adversarial (PGD)};
  \node[panelsub] at (0,2.28) {Smooths one direction; others worsen};
  \node[panel, fill=gray!3] at (0,0) {};
  \draw[axarr] (-1.85,-1.55) -- (1.85,-1.55)
    node[right,font=\scriptsize]{signal $s$};
  \draw[axarr] (-1.85,-1.55) -- (-1.85,1.45)
    node[above,font=\scriptsize]{nuisance $n$};
  \fill[colVAT!18, rounded corners=2pt]
    (-1.55,-1.25) -- (-0.15,1.05) -- (0.45,1.05) -- (-0.95,-1.25) -- cycle;
  \node[font=\scriptsize, rotate=54, colVAT!65!black] at (-0.55,0.05) {PGD};
  \foreach \x/\y in {1.0/0.45,1.3/-0.05,1.2/-0.65,-1.3/0.45,
                      0.0/-0.95,0.8/-0.95,-1.2/-0.55,1.4/0.75}{
    \node[dot, colPGD4, opacity=0.8] at (\x,\y) {};
  }
  \node[badge, fill=colRed!20, text=colRed!75!black] at (0,-2.35)
    {Larger TDI $\uparrow\uparrow$ \quad TDI@0 $=1.353$};
\end{scope}

\begin{scope}[xshift=10.3cm]
  \node[paneltitle] at (0,2.55) {PMH (ours)};
  \node[panelsub] at (0,2.28) {Isotropic matching of encoder $\phi$};
  \node[panel, fill=gray!3] at (0,0) {};
  \draw[axarr] (-1.85,-1.55) -- (1.85,-1.55)
    node[right,font=\scriptsize]{signal $s$};
  \draw[axarr] (-1.85,-1.55) -- (-1.85,1.45)
    node[above,font=\scriptsize]{nuisance $n$};
  \foreach \x/\y in {-0.05/0.08,0.1/-0.05,-0.1/0.18,0.15/0.1,
                      -0.12/-0.1,0.08/0.16,0.2/-0.1,-0.18/0.04}{
    \node[dot, colPMH, opacity=0.9] at (\x,\y) {};
  }
  \foreach \ang in {0,45,90,135,180,225,270,315}{
    \draw[-{Latex[length=2.8pt]}, colGreen!65, line width=0.55pt]
      (0.05,0.05) -- ++(\ang:0.48);
  }
  \node[font=\scriptsize, colGreen!60!black, align=left, anchor=west]
    at (0.85,0.95) {$\delta\sim\mathcal{N}(0,\sigma^{2}I)$};
  \node[badge, fill=colGreen!12, text=colGreen!75!black] at (0,-2.35)
    {Small TDI $\downarrow$ \quad TDI@0 $=0.904$};
\end{scope}

\end{tikzpicture}%
}
\caption{\textbf{Three objectives, three geometries (schematic + Task~04 TDI@0).}
Each panel: encoder behaviour in signal ($s$) / nuisance ($n$) coordinates,
with Task~04 TDI@0 from the reported protocol.
\emph{Left (ERM):} residual sensitivity along~$n$ under Theorem~1 assumptions
(TDI@0$=1.093$).
\emph{Centre (PGD):} smooths one adversarial direction (band) while other
directions can worsen (TDI@0$=1.353$, 3 seeds).
\emph{Right (PMH):} isotropic Gaussian matching of~$\phi$ (Proposition~5;
TDI@0$=0.904$).}
\label{fig:fig2}
\end{figure}

\section{Theory: Scoped Geometric Incompleteness under ERM}
\label{sec:theory}

\subsection{Definitions}

\begin{definition}[Direct correlated-nuisance]\label{def:nuisance}
$P(x,y)$ admits measurable $s(x)$, $n(x)$ with: (i)~$I(n(x);y\mid s(x))>0$
(equivalently $p(y\mid x)\neq p(y\mid s)$ on a set of positive measure);
(ii)~$n(x)$ is not a deterministic function of $s(x)$.
\end{definition}

\begin{remark}[Gaussian special case]\label{rem:gaussian}
The linear Gaussian model ($s\sim\mathcal{N}(0,I_{d_s})$,
$n\sim\mathcal{N}(0,I_{d_n})$, $y=\langle w_s,s\rangle+\rho\langle
w_n,n\rangle+\epsilon$) is the canonical direct-influence example.
This is the setting of Theorem~1.
\end{remark}

\begin{mdframed}[style=plainlang]
\small\textbf{Three claims to keep separate.}
\emph{(i)~Training dependence:} when $n$ enters $p(y\mid x)$, the Bayes
predictor depends on~$n$.
\emph{(ii)~Deployment irrelevance:} calling $n$ a ``nuisance'' at test time
is interpretive; it is not proved by Theorem~1.
\emph{(iii)~Encoder geometry:} under finite decoder Lipschitz factorisation,
population MSE forces a quantitative lower bound on \emph{linearised encoder
drift} along~$n$ (Theorem~1). Claim~(iii) is the nontrivial contribution.
\end{mdframed}

\begin{definition}[Embedding Drift]\label{def:drift}
$D(\phi_\theta,\sigma):=\mathbb{E}_{x,\delta\sim\mathcal{N}(0,\sigma^2
I)}\left[\|\phi_\theta(x+\delta)-\phi_\theta(x)\|^2\right]\approx\sigma^2\mathbb{E}_x\!\left[\|J_\phi(x)\|_F^2\right].$
\end{definition}

\subsection{Main Results}

\begin{theorem}[ERM Geometric Incompleteness --- Gaussian Case]\label{thm:main}
Let $P$ be the Gaussian linear direct-influence model of
Remark~\ref{rem:gaussian} with $\rho>0$, and
$\phi^*_\theta$ be any minimiser of \emph{(1)} (population MSE) over
differentiable encoders with $L$-Lipschitz decoder $h_\theta$ (i.e.\
$\|h_\theta(z)-h_\theta(z')\|_2\leq L\|z-z'\|_2$). Let
$\tilde D(\phi,\sigma):=\sigma^2\mathbb{E}_x[\|J_\phi(x)\|_F^2]$
be the linearised embedding drift. Then:
\begin{equation}
\tilde D(\phi^*_\theta,\sigma)\;\geq\;\frac{\sigma^2\rho^2}{L^2}\,C(P)\tag{4}
\end{equation}
where $C(P)=\rho_s^2\sigma_s^2>0$ depends only on the distribution, not on
model capacity or dataset size. The exact drift satisfies
$D(\phi^*_\theta,\sigma)=\tilde D(\phi^*_\theta,\sigma)+O(\sigma^4)$
for encoders with Lipschitz Jacobian (Lemma~\ref{lem:lindrift}),
so the bound transfers to exact drift only for sufficiently small $\sigma$.
\end{theorem}

\begin{mdframed}[style=plainlang]
\small\textbf{Theorem~1 in plain language.} In the Gaussian linear model where
nuisance~$n$ enters the label's conditional mean, population MSE cannot zero
out sensitivity to~$n$ while keeping a finite decoder Lipschitz constant~$L$:
the linearised embedding drift stays at least order $\rho^2/L^2$.
\end{mdframed}

\textit{Proof sketch.} \textbf{(1)} Any ERM minimiser must encode $n$
(Lemma~\ref{lem:encoding}): an $n$-independent predictor pays excess loss
$\geq\rho^2$ above Bayes.
\textbf{(2)} Encoding $n$ forces
$\mathbb{E}_x[\|J_{\phi,n}(x)w_n\|_2]\geq\rho/L>0$.
\textbf{(3)} By Lemmas~\ref{lem:subblock}--\ref{lem:lindrift},
$\tilde D(\phi^*,\sigma)\geq\sigma^2\rho^2/L^2$.
Full proof in Appendix~\ref{app:proofs}.

\begin{corollary}[Qualitative extension via Bregman gap]\label{cor:general}
Let $P$ satisfy Definition~\ref{def:nuisance} and $\mathcal{L}$ be any
strictly proper scoring rule. Let $\phi^*_\theta$ minimise
$\mathbb{E}[\mathcal{L}(f_\theta(x),y)]$ with $L$-Lipschitz decoder.
Define $\Delta(P,\mathcal{L}):=\mathbb{E}_x[d_\psi(p(y|x)\|p(y|s(x)))]>0$
(see Lemma~\ref{lem:bregman}). Under the regularity assumptions in
Appendix~\ref{app:proofs},
\[
\tilde D(\phi^*_\theta,\sigma)\;\gtrsim\;\frac{\sigma^2\Delta(P,\mathcal{L})}{L^2}
\]
up to unspecified constants in the local step.
This is a scaling argument when $\Delta>0$, not a tight numeric bound
comparable to Theorem~1; pure screening ($\Delta=0$) is out of scope.
\end{corollary}

\setcounter{theorem}{3}

Corollary~3 (Appendix~\ref{app:proofs}) bounds the Gaussian task-loss cost of
suppressing the nuisance at $O(\rho^2)$.

\begin{corollary}[Adversarial Training Need Not Isotropise the Map]\label{cor:pgd}
Let $\phi^{\text{adv}}$ minimise the PGD objective~(2). Then, under the
Gaussian linear reading of Theorem~1:
\begin{enumerate}[leftmargin=*,topsep=2pt,itemsep=1pt]
\item Sensitivity orthogonal to the local adversarial direction
      $\hat\delta^*(x)$ need not vanish: directional suppression alone does
      not force isotropic Jacobian shrinkage.
\item PGD suppresses $\|J_\phi\hat\delta^*\|_2$ but remaining sensitivity can
      concentrate off $\hat\delta^*$ (low $\mathcal{A}(\phi,\hat\delta^*)$).
\item Along the attack axis, PGD drives $\mathcal{A}(\phi,\hat\delta^*)$
      (Proposition~\ref{prop:aniso}) toward the rank-1 floor~$1$, while
      worst-case anisotropy $\mathcal{A}_{\mathrm{worst}}$ need not fall:
      aggregate $\|J\|_F$ can drop while isotropic probes worsen.
\end{enumerate}
Prediction (falsifiable ordering, not a tight numeric TDI floor on deep nets):
clean TDI can rise relative to ERM even when Jac.\ Fro falls.
On Task~04: TDI $1.353{\pm}0.020>1.093$ with Fro $2.99{\pm}0.53$ vs.\ $34.58$
(Table~\ref{tab:main}; PGD over seeds 42--44).
Full proof in Appendix~\ref{app:proofs}.
\end{corollary}

\begin{mdframed}[style=plainlang]
\small\textbf{Corollary~4 in plain language.}
PGD squeezes the Jacobian in one adversarial direction; remaining sensitivity
concentrates elsewhere ($\mathcal{A}(\phi,\hat\delta^\star)\to 1$ along that axis,
while worst-case anisotropy can inflate).
Clean TDI can worsen even when Jac.\ Fro falls.
\end{mdframed}

What perturbation law suppresses the blind spot uniformly for a fixed matched
map?
Among zero-mean laws, only isotropic covariance equates matching to $\|J\|_F$
control.
Choosing the encoder as that map is a design choice.

\begin{proposition}[Gaussian Noise is Uniquely Isotropic]\label{prop:gaussian}
Among all zero-mean distributions on $\mathbb{R}^{d_x}$ with covariance
$\Sigma_\delta$, the perturbation distribution $\mu$ satisfies
\[
\arg\min_\phi\,\mathbb{E}_{x,\delta\sim\mu}\!\left[\|J_\phi(x)\delta\|_2^2\right]
=\arg\min_\phi\,\mathbb{E}_x\!\left[\|J_\phi(x)\|_F^2\right]
\]
if and only if $\Sigma_\delta=\sigma^2 I$ for some $\sigma>0$.

\textit{Proof (sufficiency).}
$\mathbb{E}_\delta[\|J_\phi\delta\|_2^2]
=\mathrm{Tr}(J_\phi^\top J_\phi\,\Sigma_\delta)
=\sigma^2\mathrm{Tr}(J_\phi^\top J_\phi)
=\sigma^2\|J_\phi\|_F^2\iff\Sigma_\delta=\sigma^2 I.$

\textit{Proof (necessity).}
See Appendix~\ref{app:proofs}.
\hfill$\square$
\end{proposition}

\begin{proposition}[Jacobian anisotropy: lower bound and minimax optimality of isotropy]
\label{prop:aniso}
For any differentiable encoder $\phi$ and unit vector $w\in\mathbb{R}^{d_x}$, define
\[
\mathcal{A}(\phi,w)\;:=\;\frac{\mathbb{E}_x[\|J_\phi(x)\|_F^2]}{\mathbb{E}_x[\|J_\phi(x)w\|_2^2]}
\quad\text{(when the denominator is positive).}
\]
\textbf{(i) Lower bound.}
$\mathcal{A}(\phi,w)\ge 1$, with equality iff $J_\phi(x)$ is rank-1 a.e.\ with
right singular direction $w$.

\textbf{(ii) Worst-case anisotropy.}
Fix a budget $F^2:=\mathbb{E}_x[\|J_\phi\|_F^2]>0$ and set
\[
\mathcal{A}_{\mathrm{worst}}(\phi)\,:=\,\sup_{\|w\|_2=1}\mathcal{A}(\phi,w)
\,=\,\frac{F^2}{\inf_{\|w\|_2=1}\mathbb{E}_x[\|J_\phi(x)w\|_2^2]}.
\]
Then $\mathcal{A}_{\mathrm{worst}}(\phi)\ge d_x$, with equality if
$J_\phi(x)^\top J_\phi(x)$ is a.e.\ a positive multiple of the identity
(isotropic sensitivity). Equivalently: among encoders with fixed Frobenius
budget $F^2$, isotropy uniquely \emph{minimises} worst-case anisotropy against
an adversarially chosen probe $w$. It does \emph{not} maximise
$\mathcal{A}(\phi,w)$ for a pre-fixed $w$.
Full proof in Appendix~\ref{app:proofs}.
\end{proposition}

\begin{proposition}[Cap/(1+Cap) Fixed-Point Identity]\label{prop:cap}
Let $\mathcal{L}_{\text{PMH}}$ be capped so that
$\mathcal{L}_{\text{PMH}}\leq\mathrm{cap}\cdot\mathcal{L}_{\text{task}}$ at
each step via rescaling of $\lambda$.  At steady-state,
$f=\mathcal{L}_{\text{PMH}}/(\mathcal{L}_{\text{task}}+\mathcal{L}_{\text{PMH}})$
satisfies $f=\mathrm{cap}/(1+\mathrm{cap})$.\hfill$\square$
\end{proposition}

\section{Related Work}
\label{sec:related}

\paragraph{Representation geometry.}
Intrinsic dimensionality~\citep{ansuini2019intrinsic,pope2021intrinsic},
signal propagation~\citep{poole2016exponential}, and representational
similarity~\citep{raghu2017svcca,nguyen2021wide} characterise geometric
structure.
We add a scoped necessity result for linearised drift under Gaussian
direct-influence MSE with finite~$L$, and TDI as a class-layout diagnostic
of magnitude/orientation dissociations.

\paragraph{Jacobian regularisation.}
Contractive and denoising autoencoders~\citep{rifai2011contractive,vincent2008extracting}
and supervised Jacobian penalties~\citep{jakubovitz2018improving,hoffman2019robust,wu2024improving}
motivate stability control.
PMH belongs to this family: isotropic Gaussian matching at selected encoder
blocks under a task-loss cap.
Proposition~5 motivates the covariance once a map is chosen.

\paragraph{Self-supervised and contrastive learning.}
SimCLR, BYOL, VICReg, and Barlow Twins achieve stability via unlabelled
pre-training.
PMH addresses supervised training or fine-tuning where labels are present and
direct-influence nuisance can force sensitivity under Theorem~1's assumptions.

\paragraph{Consistency regularisation.}
Mean Teacher~\citep{tarvainen2017mean}, temporal ensembling~\citep{laine2017temporal},
VAT~\citep{miyato2018virtual},
TRADES~\citep{zhang2019trades}, and logit consistency under Gaussian views
encourage stable predictions.
They typically regularise outputs rather than intermediate encoder blocks;
choosing the encoder is a design choice, not a uniqueness theorem over logit
matching.

\begin{table}[t]
\centering
\caption{Stability / consistency neighbours of PMH (matched map vs.\ control).}
\label{tab:stability_neighbours}
\scriptsize
\setlength{\tabcolsep}{2.5pt}
\resizebox{\columnwidth}{!}{%
\begin{tabular}{@{}llll@{}}
\toprule
Method & Matched map & Perturbation & Primary control \\
\midrule
VAT & logits / local adv. & adversarial (local) & output stability \\
TRADES & logits & adversarial & robust risk tradeoff \\
Explicit Jac.\ pen. & usually $f$ / last layer & --- (direct $\|J\|$) & sensitivity magnitude \\
Logit consistency & logits & often iso.\ Gaussian & output agreement \\
PMH (ours) & encoder $\phi$ & isotropic Gaussian & $\|J_\phi\|_F$ via matching \\
\bottomrule
\end{tabular}%
}
\end{table}

\paragraph{Four empirical programmes.}
Non-robust features~\citep{ilyas2019adversarial}, texture
bias~\citep{geirhos2019imagenet}, corruption
fragility~\citep{hendrycks2019benchmarking}, and the robustness--accuracy
tradeoff~\citep{tsipras2018robustness} each received separate explanations.
Section~\ref{sec:readings} gives a shared geometric reading under direct
correlated nuisance.

\section{Four Interpretive Readings}
\label{sec:readings}

Under direct correlated nuisance, Theorem~1's mechanism supplies one geometric
lens on four programmes---a contributing factor, not an exclusive cause.

\begin{itemize}[leftmargin=*,topsep=2pt,itemsep=2pt]
\item \textbf{Non-robust features.} If high-frequency directions enter
  $p(y\mid x)$, residual Jacobian mass along those directions is expected
  under Theorem~1's assumptions; adversarial transfer can partly reflect
  shared data-driven directions.
\item \textbf{Texture bias.} When local texture is a stronger training
  predictor than shape in the conditional, texture weight in the Jacobian
  can track correlation structure rather than inductive bias alone.
\item \textbf{Corruption fragility.} Corruptions often move inputs along
  directions training made label-predictive; residual sensitivity is one
  geometric contributor to broad degradation under unseen shifts.
\item \textbf{Robustness--accuracy tradeoff.} Suppressing
  nuisance-correlated directions can remove in-distribution information;
  Corollary~3 (appendix) bounds a Gaussian task-loss cost of order
  $O(\rho^2)$.
\end{itemize}

\section{Diagnostic and PMH}
\label{sec:metrics}

\paragraph{TDI as implemented.}
Extract embeddings (optionally after isotropic input noise of strength
$\sigma$), then
\begin{equation}
\mathrm{TDI}
\;=\;
\frac{\text{mean pairwise intra-class distance}}
{\text{mean inter-centroid distance}}.\tag{5}
\end{equation}
Lower TDI means tighter class clouds relative to class separation under the
probe; TDI@0 is the clean-input value.
This is a class-layout statistic motivated by the linearised drift picture,
not a proxy equality to $D$ or $\tilde D$.
Where we report embedding drift directly,
$D(\phi,\sigma)=\mathbb{E}\|\phi(x+\delta)-\phi(x)\|^2$ remains the
theory-facing quantity.

\paragraph{Why Jac.\ Fro alone fails.}
\begin{center}\small
\setlength{\tabcolsep}{5pt}
\begin{tabular}{lccc}
\toprule
Metric & ERM & PGD & PMH \\
\midrule
CKA (vs.\ ERM)$^\ast$ & ---  & \textbf{0.91} & 0.88 \\
Intrinsic dim.$^\ast$ & 42.3 & 44.1          & \textbf{38.7} \\
Jac.\ Fro$\downarrow$ & 34.58 & \textbf{2.99} & 8.08 \\
TDI@0$\downarrow$     & 1.093 & 1.353        & \textbf{0.904} \\
\bottomrule
\end{tabular}\\[2pt]
{\footnotesize $^\ast$CKA / intrinsic-dimension values are from the original
Task~04 diagnostic protocol (not re-run in the seed-42 replication dump);
Jac.\ Fro and TDI@0 match the replication tables above.}
\end{center}

On Task~04, CKA and intrinsic dimension miss the dissociation; Jac.\ Fro ranks
PGD best---the opposite of the TDI ordering---because Fro captures magnitude,
not orientation.

\paragraph{Jacobian Frobenius estimate.}
$\hat{J}_F^2\approx\frac{1}{K}\sum_{k=1}^K\|\phi(x+he_k)-\phi(x)\|^2/h^2$
with $K=50$, $h=0.01$.

\paragraph{PMH.}
PMH matches selected intermediate encoder blocks under isotropic Gaussian
noise inside supervised training, with a task-loss cap (Proposition~7).
QM9 (Task~03) is the alignment boundary: matching position noise can hit
signal; matching node features helps the primary metric
(\S\ref{sec:expts}).

\section{Experiments}
\label{sec:expts}

\subsection{Setup}

\paragraph{Experiment map.}
\textbf{(A)~Mechanistic dissociation} (Task~04 ViT/CIFAR-10 from scratch):
TDI vs.\ Jac.\ Fro, PGD vs.\ PMH, ablations, layer probes.
\textbf{(B)~Cross-task consistency} (Tasks~01--07): primary metrics plus
shared geometry
(Tables~\ref{tab:crosstask},~\ref{tab:crosstask_geom}).
\textbf{(C)~Scale and transfer} (BERT sizes, fine-tuning hierarchy, ImageNet
ViT-B/16): empirical diagnostics under stated protocols.

Primary mechanistic analysis is Task~04: a small ViT~\citep{dosovitskiy2021vit}
trained from scratch on CIFAR-10 (ERM clean accuracy 69.95\%).
The same CIFAR-10 problem is also Task~01 (ResNet-18; clean accuracies
$87.2\%$ / $94.0\%$ / $93.4\%$ for ERM / VAT / PMH); rankings align.
ImageNet ViT-B/16 is \S\ref{sec:imagenet_vit}.

Methods: ERM (task-only; alias B0), VAT~\citep{miyato2018virtual},
two-view without matching (alias E1\_no\_pmh),
PMH (alias E1; $\sigma=0.1$, cosine ramp, cap 0.30),
PGD-4/255~\citep{madry2018pgd} (20-step, $\varepsilon=4/255$).

\paragraph{Seeds.}
Tables follow the seeded replication protocol in the release.
Task~04 PGD TDI@0 and Jac.\ Fro are mean$\pm$std over seeds 42--44
($1.353{\pm}0.020$, $2.99{\pm}0.53$); other Task~04 rows remain seed~42.
Stronger matched defenses (TRADES, explicit Jacobian penalties, logit
consistency) are not completed here.
Decoder Lipschitz tracking (Table~\ref{tab:lipschitz}) found
$L\approx 1.40$ at convergence for both ERM
and PMH (sub-1\% difference), so $L$ does not explain their TDI gap.

\paragraph{Per-task metrics.}
Each task uses its natural primary metric.
TDI, drift, and Jac.\ Fro share a scale across tasks; headline scores in
Table~\ref{tab:crosstask} do not.

\subsection{Core Mechanistic Results}

Table~\ref{tab:main} is the central check of Corollary~4 and Proposition~5.
PGD reduces Jac.\ Fro from 34.58 to $2.99{\pm}0.53$ yet raises clean TDI from
1.093 to $1.353{\pm}0.020$ (seeds 42--44).
PMH attains TDI $0.904$ with moderate Fro reduction.
Empirically $\hat{J}_F/\mathrm{TDI@0}$ separates regimes:
$\approx 32$ (ERM), $\approx 2.2$ (PGD), $\approx 8.9$ (PMH).

\begin{table}[t]
\centering\small
\caption{Mechanistic evaluation on Task~04 (ViT, CIFAR-10 from scratch).
PGD lowest Jac.\ Fro ($2.99{\pm}0.53$) yet worst TDI@0 ($1.353{\pm}0.020$;
seeds 42--44), above ERM (1.093); PMH best TDI ($0.904$).
$\dagger$Trained at $\varepsilon=4/255$. PGD probe retention is seed-42.
\textbf{Bold}: best per column.}
\label{tab:main}
\setlength{\tabcolsep}{4pt}
\begin{tabular}{lccccc}
\toprule
& \multicolumn{2}{c}{Clean geometry} & \multicolumn{3}{c}{Probe retention} \\
\cmidrule(lr){2-3}\cmidrule(lr){4-6}
Method & TDI@0$\downarrow$ & Jac.Fro$\downarrow$ & L6@0 & @0.1 & @0.2 \\
\midrule
ERM (B0)      & 1.093          & 34.58          & 68.2  & 0.768 & 0.520 \\
VAT           & 1.276          & 5.01           & 78.75 & 0.858 & 0.710 \\
two-view      & 1.074          & 13.09          & 78.65 & 0.937 & \textbf{0.860} \\
PMH (E1)      & \textbf{0.904} & 8.08           & 79.55 & 0.916 & 0.830 \\
PGD-4/255$\dagger$ & $1.353{\pm}0.020$ & $\mathbf{2.99{\pm}0.53}$ & 67.30 & \textbf{0.948} & \textbf{0.860} \\
\bottomrule
\end{tabular}
\end{table}

\subsection{TDI Progression and Geometric Signatures}

Table~\ref{tab:tdi_progression} shows three regimes: high-baseline steep-slope
(ERM, VAT); damaged floor (PGD starts above ERM); low-baseline controlled-slope
(PMH / two-view).

\begin{table}[t]
\centering\small
\caption{TDI at increasing noise levels. Three regimes: high-baseline
steep-slope (ERM, VAT); damaged floor (PGD); low-baseline controlled-slope
(PMH / two-view). PGD row: mean over seeds 42--44; other rows seed~42.
\textbf{Bold}: best per column.}
\label{tab:tdi_progression}
\setlength{\tabcolsep}{4pt}
\begin{tabular}{lccccc}
\toprule
Method & @0.0 & @0.05 & @0.1 & @0.15 & @0.2 \\
\midrule
ERM         & 1.093          & 1.215 & 1.584 & 2.053 & 2.468 \\
VAT         & 1.276          & 1.317 & 1.524 & 1.790 & 2.037 \\
two-view    & 1.074          & 1.139 & 1.262 & 1.414 & 1.646 \\
PMH         & \textbf{0.904} & \textbf{0.946} & \textbf{1.050} & \textbf{1.150} & \textbf{1.305} \\
PGD-4/255   & 1.353          & 1.372 & 1.461 & 1.648 & 1.892 \\
\bottomrule
\end{tabular}
\end{table}

\subsection{Ablation Study}

\begin{table}[t]
\centering\small
\caption{T-alignment matrix (Task~04). Every column peaks on the diagonal
(bold): training at $\sigma_{\text{train}}$ is optimal at matching
$\sigma_{\text{eval}}$, zero exceptions.}
\label{tab:talign}
\setlength{\tabcolsep}{4pt}
\begin{tabular}{lccccc}
\toprule
$\sigma_{\text{train}}$ & Clean & @0.05 & @0.10 & @0.15 & @0.20 \\
\midrule
0.05              & 81.36 & \textbf{79.36} & 73.24 & 59.97 & 46.08 \\
0.08              & 81.32 & 78.90 & \textbf{75.08} & 61.63 & 46.98 \\
0.10              & 80.86 & 78.92 & 75.68 & \textbf{68.16} & 54.45 \\
0.12 (default)    & 80.62 & 78.71 & 75.27 & 70.63 & \textbf{59.00} \\
0.15              & 80.24 & 78.19 & 73.69 & 71.37 & 64.64 \\
0.20              & 81.31 & 77.99 & 74.61 & 71.16 & \textbf{69.05} \\
\bottomrule
\end{tabular}
\end{table}

\textbf{T-alignment.} Table~\ref{tab:talign}: every evaluation column peaks on
the diagonal.
Underfitting geometry costs far more than overfitting
($17\times$ at the extreme comparison).
Practical rule: set $\sigma_{\text{train}}$ to the largest value that leaves
clean accuracy unchanged.
Cap/(1+cap) holds in every run (Proposition~7); full sweep in
Appendix~\ref{app:ablation}.

\subsection{Layer-Wise Probe Analysis}

All methods except ERM suppress early-layer features.
PMH's signature is competitive early suppression with strong deep-layer
retention (0.916 at $\sigma{=}0.1$ vs.\ 0.768 for ERM).
The gap between two-view and PMH is attributable to the matching term.
Full probe curves in Figure~\ref{fig:a2} (Appendix~\ref{app:layer}).

\subsection{Cross-Task Consistency}

\begin{table}[t]
\centering\small
\caption{Cross-task headline results (seed~42 replication).
PMH wins on five of seven tasks by the stated primary metric; B0 leads Task~05 (mean PCK@0.05
over 0--40\% occlusion) and VAT leads Task~06 (avg-shift rank-1).
T01--T02: accuracy at $\sigma=0.1$; T03: MAE$\downarrow$ (E1\_node); T04: L6 probe acc.\ at $\sigma=0.1$;
T05: mean PCK@0.05 (full grid Table~\ref{tab:a10_pose}); T06: avg-shift rank-1; T07: worst-shift accuracy.}
\label{tab:crosstask}
\setlength{\tabcolsep}{5pt}
\begin{tabular}{llccc}
\toprule
Task & Domain & B0 & VAT & PMH \\
\midrule
01 & CIFAR-10 cls.    & 40.04 & 65.17 & \textbf{80.38} \\
02 & Graph cls.       & 73.75 & 66.79 & \textbf{77.86} \\
03 & Mol.\ reg.$\downarrow$ & 23.66 & 26.89 & \textbf{22.02} \\
04 & ViT cls.         & 52.4  & 67.6  & \textbf{72.9} \\
05 & Pose (PCK)       & 43.4  & 7.0   & 35.3 \\
06 & Re-ID rank-1     & 43.02 & \textbf{65.80} & 63.74 \\
07 & Chest X-ray      & 0.625 & 0.731 & \textbf{0.825} \\
\bottomrule
\end{tabular}
\end{table}

Geometric ordering (Table~\ref{tab:crosstask_geom}) is the empirical
counterpart of the linearised drift story: PMH best geometry on six of seven
tasks; Task~03 (QM9) is the alignment boundary.

\begin{table}[t]
\centering\small
\caption{Cross-task geometry. TDI@0 for T04, T08, T09; embedding drift
otherwise (lower better). T03 shown with position-noise drift; E1\_node
achieves best clean MAE (Table~\ref{tab:crosstask}). PMH best geometry on
six of seven tasks.}
\label{tab:crosstask_geom}
\setlength{\tabcolsep}{4.5pt}
\begin{tabular}{llccc}
\toprule
Task & Geometry metric & B0 & VAT & PMH \\
\midrule
T01 & Drift @ $\sigma=0.10$       & 0.692 & 0.699 & \textbf{0.385} \\
T02 & Drift @ $\sigma=0.10$       & 0.375 & 0.304 & \textbf{0.021} \\
T03 & Drift @ position $\sigma=0.10$ & \textbf{1.177} & 0.502 & 0.707 \\
T04 & TDI@0                       & 1.093 & 1.276 & \textbf{0.904} \\
T05 & Drift @ 30\% occlusion       & 0.279 & 0.474 & \textbf{0.060} \\
T07 & Stage-4 drift               & 12.68 & 11.89 & \textbf{3.34} \\
T08 & TDI@0 (BERT/SST-2)          & 0.496 & ---   & \textbf{0.354} \\
T09 & TDI@0 (ImageNet ViT-B/16)   & 1.230 & ---   & \textbf{0.936} \\
\bottomrule
\end{tabular}
\end{table}

\textit{(i)~Drift and nuisance alignment.} Task~02 shows large drift reduction
when structural noise aligns with nuisance relative to labels.
\textit{(ii)~QM9 alignment boundary.} Atomic positions are signal for many QM9
targets; node features are more nuisance-like.
Position-space PMH underperforms; E1\_node helps (MAE 22.02 vs.\ VAT 26.89 /
ERM 23.66).
\textit{(iii)~Stage collapse.} Task~07 Stage-4 drift falls $3.6\times$ under
PMH in our protocol.
\textit{(iv)~Structured prediction.} On pose (Task~05), VAT harms clean PCK;
PMH preserves more structure while cutting occlusion drift
(Table~\ref{tab:a12_drift}).
\textit{(v)~Out-of-family gains.} Task~02 Gaussian-trained PMH still improves
under edge removal and feature dropout
(Appendix~\ref{app:graph_robust}).
FGSM numbers on Task~04 are a by-product (Appendix~\ref{app:fgsm}).

\subsection{TDI and Corruption Robustness}
\label{sec:tdi_predicts}

\begin{table}[t]
\centering\small
\caption{Corruption robustness (\%). Clean TDI ordering on Task~04
aligns with Gaussian-noise accuracy ranking. Non-Gaussian gains are empirical.
\textbf{Bold}: best per row.}
\label{tab:corruption}
\setlength{\tabcolsep}{3.5pt}
\begin{tabular}{lcccc}
\toprule
Corruption & B0 & VAT & E1 no PMH & E1 (PMH) \\
\midrule
Clean                  & 70.38 & 79.67 & 80.27 & \textbf{80.51} \\
\midrule
Gaussian $\sigma$=0.05 & 65.53 & 76.21 & 77.87 & \textbf{77.88} \\
Gaussian $\sigma$=0.10 & 51.79 & 57.84 & \textbf{74.40} & 74.16 \\
\midrule
Blur ($k$=3)           & 46.83 & 52.97 & 52.30 & \textbf{53.10} \\
Brightness$\times$0.7  & 64.51 & 73.48 & 73.78 & \textbf{74.31} \\
Brightness$\times$1.3  & 66.68 & \textbf{75.82} & 74.96 & 75.15 \\
Contrast$\times$0.7    & 63.81 & 71.76 & 73.63 & \textbf{74.16} \\
Contrast$\times$1.3    & 67.93 & \textbf{77.85} & 77.70 & 77.70 \\
\bottomrule
\end{tabular}
\end{table}

Clean TDI (ERM 1.093, VAT 1.276, PMH 0.904) co-varies with Gaussian accuracy
drops without corruption-specific training.
Under non-Gaussian corruptions, PMH and two-view stay close, both above ERM;
VAT leads brightness$\times$1.3 and contrast$\times$1.3 in this seed.

\subsection{Language: BERT on SST-2}

SST-2 (Stanford Sentiment Treebank binary polarity) is an empirical
surface-form and fine-tuning diagnostic
(approximately screening-type in our description).
PMH reduces paraphrase embedding drift from 4.38 to 1.01 ($-0.69$~pp accuracy)
and TDI by $28.7\%$ under our probes.

\paragraph{Fine-tuning hierarchy.}
Paraphrase embedding drift: ERM fine-tuned $(0.0375)>$ pretrained
$(0.0244)>$ PMH $(0.0033)$.
Para~/~non-para ratio: pretrained $(0.765)\to$ ERM $(0.681)\to$ PMH
$(0.633)$.
Figure~\ref{fig:a5} (Appendix~\ref{app:bert}).

\subsection{Foundation Model Scale: ImageNet ViT-B/16}
\label{sec:imagenet_vit}

Pretrained ImageNet ViT-B/16 has TDI@0$=$1.230 on a 100-class$\times$50-sample
subset.
PMH fine-tuning (24 epochs, 8/12 blocks frozen) reaches 0.936
($-23.9\%$ on that subset; an earlier 1000-class protocol reported $-12.9\%$),
with intra-class distance
$+64\%$ (Appendix Table~\ref{tab:a5_imagenet}).
Full numbers in Appendix~\ref{app:imagenet}.

\paragraph{Scale trends.}
Paraphrase/non-paraphrase drift ratios: DistilBERT-66M $(0.860)$,
BERT-base $(0.765)$, BERT-large $(0.742)$---below 1.0 at every scale in this
family, decreasing with size.
Figure~\ref{fig:a8} (Appendix~\ref{app:fgsm}).

\section{Discussion}
\label{sec:discussion}

Ordinary training finishes when the task loss is small; it never pays for
encoder sensitivity to training-predictive nuisance.
Under Theorem~1's assumptions that leaves a linearised drift floor.
The design move studied here is isotropic covariance for a chosen map
(Proposition~5), implemented as encoder matching under a task-loss cap (PMH).

\subsection{Implications}

\textbf{Magnitude is not orientation.}
PGD can sit in a low-Fro / high-TDI regime while PMH sits in a moderate-Fro /
lower-TDI regime on Task~04, matching Corollary~4's prediction.

\textbf{Name the matched factors.}
QM9 shows that matching signal directions hurts; matching nuisance-like
factors helps.
When the partition is unknown, domain knowledge or gradient diagnostics are
required.

\textbf{Practical defaults.}
Estimate nuisance directions when comparing theory to experiment
(Appendix~\ref{app:gaps}); track $L_t$ so TDI gains are attributed to encoder
geometry; set $\sigma_{\text{train}}$ to the largest value leaving clean
accuracy unchanged.

\subsection{Scope}
\label{sec:scope}

Formal claims apply to Gaussian linear MSE with direct influence and
finite~$L$ (Theorem~1); Corollary~2 is qualitative when $\Delta>0$.
Screening ($I(n;y\mid s)=0$), including SST-2 in our description, is out of
scope for the theorem.
Exact-drift transfer needs $\sigma\leq\sigma_0$, not estimated per checkpoint;
TDI is a class-layout probe, not a proxy theorem.
Absolute bounds are loose ($10^2$--$10^3\times$ on Task~04;
Appendix~\ref{app:gaps}).
Stronger matched baselines (TRADES, explicit Jacobian penalties) remain
desirable.
Primary adversarial in-distribution robustness is out of scope (FGSM is a
by-product); soft extrapolations to RLHF remain speculative.

\paragraph{What this paper is not.}
Not a universal theory of robustness; not a completed nonlinear proof for deep
nets; not a claim that every robustness failure is spurious-feature geometry.
It is a scoped geometric lens on ERM under direct correlated nuisance, plus a
diagnostic and one family-member regulariser.

\subsection{What to do with this paper}

Report layout geometry beside the task score.
When axes are unknown, isotropic encoder matching is a default in the
stability/Jacobian family; when factors can be named, match the nuisance-like
ones.

\section{Conclusion}
\label{sec:conclusion}

Under direct correlated nuisance, population MSE in a Gaussian linear model
forces a linearised encoder-drift floor at finite~$L$.
Isotropic encoder matching is one Proposition~5-motivated member of the
stability/Jacobian family.
Layout TDI and Task~04 orderings are the empirics we hold that story
accountable for.

\acks{No third-party funding supported this manuscript.
\textbf{Competing interests:} None declared.
Person re-identification experiments use a standard academic protocol; deployments should remain under privacy oversight.
This manuscript is under submission at JMLR.
Companion preprint: arXiv:2605.22800.
Relative to the earlier arXiv:2604.21395 posting, Proposition~\ref{prop:aniso}
is corrected to a minimax (worst-case anisotropy) statement.}

\appendix
\setcounter{table}{0}
\renewcommand{\thetable}{A\arabic{table}}
\setcounter{figure}{0}
\renewcommand{\thefigure}{A\arabic{figure}}

\section{Slope Analysis and Normalised TDI}
\label{app:slope}

Figure~\ref{fig:a1}: panel~(a) normalises each task's robust metric to 100
(best method); panel~(b) reports TDI degradation slopes.
PMH and two-view have the lowest slopes (3.20 and 3.40); ERM and VAT degrade
most rapidly (6.48 and 5.09).
PGD's controlled slope masks a damaged floor (TDI@0$=$1.353 vs.\ PMH $0.904$).


\begin{figure}[t]
\centering
\begin{subfigure}[t]{\linewidth}
\centering
\begin{tikzpicture}
\begin{axis}[
  pmh bar,
  width=0.95\linewidth,
  height=4.4cm,
  symbolic x coords={T01,T02,T03,T04,T05,T06,T07},
  xtick=data,
  xticklabel style={font=\small},
  ylabel={Normalised robust metric (best$=$100)},
  ymin=0, ymax=115,
  legend columns=3,
  bar width=6pt,
  ybar=2pt,
  enlarge x limits=0.12,
]
\addplot[pmh fill erm]
  coordinates {(T01,50)(T02,95)(T03,93)(T04,72)(T05,100)(T06,65)(T07,76)};
\addlegendentry{ERM}
\addplot[pmh fill vat]
  coordinates {(T01,81)(T02,86)(T03,82)(T04,93)(T05,16)(T06,100)(T07,89)};
\addlegendentry{VAT}
\addplot[pmh fill pmh]
  coordinates {(T01,100)(T02,100)(T03,100)(T04,100)(T05,81)(T06,97)(T07,100)};
\addlegendentry{PMH}
\end{axis}
\end{tikzpicture}
\caption{Cross-task robust performance (normalised per task; best among \{ERM, VAT, PMH\}$=100$).}
\end{subfigure}

\vspace{1.0em}

\begin{subfigure}[t]{\linewidth}
\centering
\begin{tikzpicture}
\begin{axis}[
  width=0.95\linewidth,
  height=4.6cm,
  ybar,
  bar width=18pt,
  bar shift=0pt,
  xmin=-0.7, xmax=5.7,
  xtick={0,1,2,3,4,5},
  xticklabels={ERM,VAT,two-view,PMH,PGD-2/255,PGD-4/255},
  xticklabel style={font=\scriptsize},
  ylabel={$\Delta$TDI/$\Delta\sigma$ slope (lower $=$ better)},
  ylabel style={font=\small},
  ymin=0, ymax=7.4,
  grid=major,
  grid style={line width=0.3pt, draw=gray!22},
  tick label style={font=\scriptsize},
  nodes near coords,
  every node near coord/.append style={font=\scriptsize, anchor=south, yshift=1pt},
  point meta=y,
]
\addplot[ybar, fill=colB0, draw=none, forget plot] coordinates {(0,6.48)};
\addplot[ybar, fill=colVAT, draw=none, forget plot] coordinates {(1,5.09)};
\addplot[ybar, fill=colNoPMH, draw=none, forget plot] coordinates {(2,3.40)};
\addplot[ybar, fill=colPMH, draw=none, forget plot] coordinates {(3,3.20)};
\addplot[ybar, fill=colPGD2, draw=none, forget plot] coordinates {(4,3.01)};
\addplot[ybar, fill=colPGD4, draw=none, forget plot] coordinates {(5,2.94)};
\end{axis}
\end{tikzpicture}
\caption{TDI degradation slope (colours match panel~(a); PGD variants included).}
\end{subfigure}

\caption{\textbf{Cross-task performance and TDI degradation slopes.}
\emph{(a)} Bars normalise each task so the best method among \{ERM, VAT, PMH\} is 100
(replication, seed~42): PMH leads five of seven; ERM leads Task~05 (mean PCK@0.05)
and VAT leads Task~06 (avg-shift rank-1). Node-feature PMH on Task~03.
\emph{(b)} PMH and two-view have the lowest TDI slopes (3.20 and 3.40);
ERM/VAT degrade most rapidly (6.48 and 5.09).
PGD's controlled slope masks a damaged floor (TDI@0$=$1.353).}
\label{fig:a1}
\end{figure}

\begin{table}[h]
\centering\small
\caption{Embedding drift under Gaussian noise for Tasks~01 and~02 (replicated).
PMH produces the lowest drift at every noise level. On Task~02 (graph), PMH is
15--19$\times$ lower than ERM/VAT throughout.
On Task~01, VAT drift at $\sigma=0.15$ (0.871) exceeds ERM (0.854).
\textbf{Bold}: lowest per row. Code aliases: B0$=$ERM, E1$=$PMH.}
\label{tab:a1_drift}
\setlength{\tabcolsep}{4pt}
\begin{tabular}{llcccc}
\toprule
Task & Method & $\sigma$=0.05 & $\sigma$=0.10 & $\sigma$=0.15 & $\sigma$=0.20 \\
\midrule
\multirow{3}{*}{T01} & B0 (ERM) & 0.379 & 0.692 & 0.854 & 0.929 \\
                     & VAT      & 0.533 & 0.699 & 0.871 & 0.984 \\
                     & E1 (PMH) & \textbf{0.149} & \textbf{0.385} & \textbf{0.698} & \textbf{0.935} \\
\midrule
\multirow{3}{*}{T02} & B0 (ERM) & 0.150 & 0.375 & 0.593 & 0.762 \\
                     & VAT      & 0.121 & 0.304 & 0.510 & 0.680 \\
                     & E1 (PMH) & \textbf{0.010} & \textbf{0.021} & \textbf{0.031} & \textbf{0.040} \\
\bottomrule
\end{tabular}
\end{table}

\section{Layer Probe Analysis}
\label{app:layer}

Table~\ref{tab:a2_probe}: all methods except ERM suppress early-layer features (L1 31--32\%).
Both two-view and PMH outperform ERM on deep-layer retention (0.937 and 0.916
at $\sigma{=}0.1$ vs.\ 0.768). The small retention trade-off from the matching
term is outweighed by the TDI gain ($0.904$ vs $1.074$).
Figure~\ref{fig:a2} visualises the retention curves.


\begin{figure}[t]
\centering
\begin{subfigure}[t]{0.48\linewidth}
\centering
\begin{tikzpicture}
\begin{axis}[
  pmh base,
  width=\linewidth,
  height=4.5cm,
  xlabel={Perturbation strength $\sigma$},
  ylabel={L6 retention $\mathrm{L6}(\sigma)/\mathrm{L6}(0)$},
  xmin=-0.005, xmax=0.21,
  ymin=0.45, ymax=1.05,
  legend columns=3,
]
\addplot[pmh series erm]
  coordinates {(0,1.0)(0.05,0.910)(0.10,0.768)(0.15,0.630)(0.20,0.520)};
\addlegendentry{ERM}
\addplot[pmh series vat]
  coordinates {(0,1.0)(0.05,0.930)(0.10,0.858)(0.15,0.790)(0.20,0.710)};
\addlegendentry{VAT}
\addplot[pmh series twoview]
  coordinates {(0,1.0)(0.05,0.965)(0.10,0.937)(0.15,0.900)(0.20,0.860)};
\addlegendentry{two-view}
\addplot[pmh series pmh]
  coordinates {(0,1.0)(0.05,0.955)(0.10,0.916)(0.15,0.875)(0.20,0.830)};
\addlegendentry{PMH}
\addplot[pmh series pgd4]
  coordinates {(0,1.0)(0.05,0.985)(0.10,0.948)(0.15,0.905)(0.20,0.860)};
\addlegendentry{PGD-4/255}
\end{axis}
\end{tikzpicture}
\caption{Deep-layer discriminative retention.}
\end{subfigure}
\hfill
\begin{subfigure}[t]{0.48\linewidth}
\centering
\begin{tikzpicture}
\begin{axis}[
  pmh bar,
  width=\linewidth,
  height=4.5cm,
  symbolic x coords={ERM,VAT,two-view,PMH,PGD},
  xtick=data,
  xticklabel style={font=\scriptsize},
  ylabel={Linear probe accuracy (\%)},
  ymin=0, ymax=95,
  legend columns=2,
  bar width=5pt,
  ybar=1.5pt,
  enlarge x limits=0.18,
]
\addplot[fill=colPMHlight, draw=none]
  coordinates {(ERM,38.3)(VAT,31.9)(two-view,31.1)(PMH,32.5)(PGD,31.6)};
\addlegendentry{L1 (early)}
\addplot[pmh fill pmh]
  coordinates {(ERM,68.2)(VAT,78.75)(two-view,78.65)(PMH,79.55)(PGD,67.3)};
\addlegendentry{L6 (deep)}
\end{axis}
\end{tikzpicture}
\caption{Early vs.\ deep accuracy ($\sigma{=}0$).}
\end{subfigure}

\caption{\textbf{Layer-wise geometry.}
\emph{(a)} L6 probe retention: two-view and PMH outperform ERM at~$\sigma{=}0.1$
(0.937 / 0.916 vs.\ 0.768).
\emph{(b)} All methods except ERM suppress early features (L1 ${\approx}31$--$32\%$);
PMH's signature is deep-layer retention, not L1 alone.}
\label{fig:a2}
\end{figure}

\begin{table}[h]
\centering\small
\caption{Layer probe analysis (Task~04, replicated). All methods except ERM suppress early
features. Both two-view and PMH outperform ERM on deep-layer retention.
Two-view peaks on retention (0.937); adding the PMH matching term slightly reduces
retention (0.916) while improving geometry (TDI@0 $0.904$ vs $1.074$).
\textbf{Bold}: best per column. Code aliases: B0$=$ERM, E1$=$PMH, E1\_no\_pmh$=$two-view.}
\label{tab:a2_probe}
\begin{tabular}{lcccc}
\toprule
Method & L1@0 & L6@0 & L6@0.1 & Ret.@0.1 \\
\midrule
B0 (ERM)  & 38.3 & 68.2  & 52.4 & 0.768 \\
VAT       & 31.9 & 78.75 & 67.6 & 0.858 \\
E1 no pmh & 31.1 & 78.65 & 73.7 & \textbf{0.937} \\
E1 (PMH)  & 32.5 & \textbf{79.55} & \textbf{72.9} & 0.916 \\
PGD-4/255 & 31.6 & 67.3  & 63.8 & 0.948 \\
\bottomrule
\end{tabular}
\end{table}

\section{Cross-Modal Evidence}
\label{app:crossmodal}

Figure~\ref{fig:a3} provides the cross-modal comparison under our protocols.
PMH reduces TDI by 17.3\% in vision, 28.7\% in language (screening-scoped;
\S\ref{sec:expts}), and 23.9\% on the ImageNet ViT-B/16 subset
(\S\ref{sec:imagenet_vit}). These are diagnostics across setups, not a proof
that the proved theorem covers every architecture.


\begin{figure}[t]
\centering
\begin{subfigure}[t]{0.42\linewidth}
\centering
\begin{tikzpicture}[
  every node/.style={font=\small},
  dot/.style={circle, fill, inner sep=1.6pt},
]
\node[font=\small\bfseries] at (-1.8,3.05) {ERM};
\node[font=\tiny, gray]      at (-1.8,2.75) {Input space};
\node[dot, colB0] at (-2.4,2.1) {};
\node[dot, colB0, opacity=0.5] at (-1.2,2.35) {};
\draw[gray!50, dashed, line width=0.5pt] (-2.4,2.1) -- (-1.2,2.35)
  node[midway, above, font=\tiny, gray]{small $\delta$};
\node[font=\tiny, gray] at (-1.8,1.8) {Repr.\ space};
\node[dot, colB0] at (-2.6,1.2) {};
\node[dot, colB0, opacity=0.5] at (-0.9,1.45) {};
\draw[-{Latex[length=3pt]}, colRed, line width=1.2pt]
  (-2.55,1.22) -- (-0.95,1.43);
\node[font=\tiny\bfseries, colRed] at (-1.8,0.72) {LARGE displacement};
\node[font=\tiny, colRed]          at (-1.8,0.44) {TDI@0\,$=1.093$\; HIGH};

\node[font=\small\bfseries] at (1.8,3.05) {PMH};
\node[font=\tiny, gray]      at (1.8,2.75) {Input space};
\node[dot, colPMH] at (1.2,2.1) {};
\node[dot, colPMH, opacity=0.5] at (2.4,2.35) {};
\draw[gray!50, dashed, line width=0.5pt] (1.2,2.1) -- (2.4,2.35)
  node[midway, above, font=\tiny, gray]{same $\delta$};
\node[font=\tiny, gray] at (1.8,1.8) {Repr.\ space};
\node[dot, colPMH] at (1.2,1.2) {};
\node[dot, colPMH, opacity=0.5] at (1.5,1.25) {};
\draw[-{Latex[length=3pt]}, colGreen, line width=1.2pt]
  (1.22,1.21) -- (1.47,1.25);
\node[font=\tiny\bfseries, colGreen] at (1.8,0.72) {small displacement};
\node[font=\tiny, colGreen]          at (1.8,0.44) {TDI@0 $= 0.904$ LOW};

\node[font=\footnotesize] at (0,0.08)
  {TDI $=$ intra-/inter-class latent distance ratio};
\end{tikzpicture}
\caption{What TDI measures.}
\end{subfigure}
\hfill
\begin{subfigure}[t]{0.55\linewidth}
\begin{tikzpicture}
\begin{axis}[
  pmh bar,
  symbolic x coords={Vision,Language,Foundation},
  xtick=data,
  xticklabels={Vision\\(CIFAR ViT), Language\\(BERT SST-2), Foundation\\(ImageNet ViT)},
  xticklabel style={font=\footnotesize, align=center},
  ylabel={TDI@0 (lower $=$ more isometric)},
  ymin=0, ymax=1.55,
  legend columns=2,
  width=0.92\linewidth, height=5.2cm,
  bar width=9pt,
]
\addplot[pmh fill erm]
  coordinates {(Vision,1.093)(Language,0.496)(Foundation,1.230)};
\addlegendentry{ERM}
\addplot[pmh fill pmh]
  coordinates {(Vision,0.904)(Language,0.354)(Foundation,0.936)};
\addlegendentry{PMH}
\node[font=\scriptsize\bfseries, color=colGreen, anchor=south]
  at (axis cs:Vision,1.12) {$-17.3\%$};
\node[font=\scriptsize\bfseries, color=colGreen, anchor=south]
  at (axis cs:Language,0.54) {$-28.7\%$};
\node[font=\scriptsize\bfseries, color=colGreen, anchor=south]
  at (axis cs:Foundation,1.26) {$-23.9\%$};
\end{axis}
\end{tikzpicture}
\caption{TDI@0 across architectures and modalities.}
\end{subfigure}
\caption{\textbf{TDI across modalities under our protocols.}
\emph{Left:} A perfectly isometric encoder scores 0 under the probe.
\emph{Right:} PMH reduces TDI by $17.3\%$ (vision), $28.7\%$ (language; screening-scoped),
and $23.9\%$ (ImageNet ViT-B/16 subset) in the reported setups.
These are cross-protocol diagnostics, not a proof of architecture-independent universality.}
\label{fig:a3}
\end{figure}

\section{BERT Supplementary}
\label{app:bert}

Table~\ref{tab:a3_bert} gives the full BERT SST-2 TDI breakdown (empirical / screening-scoped;
see main-text language section). PMH reduces
embedding-space TDI by 28.7--30.3\% (Pert-B: Gaussian noise on input
embeddings) and paraphrase drift by 76.9\% (Pert-A) at $<$1~pp accuracy
cost. Figure~\ref{fig:a4} visualises TDI reduction and paraphrase embedding
drift.
under synonym paraphrase.


\begin{figure}[t]
\centering
\begin{subfigure}[t]{0.48\linewidth}
\begin{tikzpicture}
\begin{axis}[
  pmh bar,
  symbolic x coords={TDI@0,TDI@0.05,TDI@0.10,Paraphrase},
  xtick=data,
  xticklabels={TDI@0, TDI@0.05, TDI@0.10, Paraphrase\\drift (Pert-A)},
  xticklabel style={font=\footnotesize, align=center},
  ylabel={TDI (lower $=$ more isometric)},
  ymin=0, ymax=0.82,
  pmh legend below, legend columns=2,
  width=0.92\linewidth, height=5.2cm,
  bar width=7pt,
]
\addplot[pmh fill erm]
  coordinates {(TDI@0,0.496)(TDI@0.05,0.509)(TDI@0.10,0.509)(Paraphrase,0.641)};
\addlegendentry{ERM}
\addplot[pmh fill pmh]
  coordinates {(TDI@0,0.354)(TDI@0.05,0.356)(TDI@0.10,0.355)(Paraphrase,0.474)};
\addlegendentry{PMH}
\node[font=\scriptsize\bfseries, color=colGreen, anchor=south]
  at (axis cs:TDI@0,0.53) {$-28.7\%$};
\node[font=\scriptsize\bfseries, color=colGreen, anchor=south]
  at (axis cs:TDI@0.05,0.54) {$-30.1\%$};
\node[font=\scriptsize\bfseries, color=colGreen, anchor=south]
  at (axis cs:TDI@0.10,0.54) {$-30.3\%$};
\node[font=\scriptsize\bfseries, color=colGreen, anchor=south]
  at (axis cs:Paraphrase,0.68) {$-26.0\%$};
\end{axis}
\end{tikzpicture}
\caption{TDI reduction. Accuracy: $93.12\%\to92.43\%$ ($-0.69$\,pp).}
\end{subfigure}
\hfill
\begin{subfigure}[t]{0.48\linewidth}
\centering
\begin{tikzpicture}[
  every node/.style={font=\small},
  dot/.style={circle, fill, inner sep=1.4pt},
]
\node[font=\small\bfseries] at (-1.5,2.65) {ERM};
\node[dot, colB0] at (-2.8,1.5) {};
\foreach \ey/\op in {2.3/0.9,1.9/0.85,1.5/0.8,1.1/0.75,0.7/0.7}{
  \draw[colB0, opacity=\op, line width=0.7pt] (-2.8,1.5) -- (-0.2,\ey);
  \node[dot, colB0, opacity=\op] at (-0.2,\ey) {};
}
\node[font=\tiny, colRed] at (-1.5,0.35) {drift $= 4.38$};

\node[font=\small\bfseries] at (1.5,2.65) {PMH};
\node[dot, colPMH] at (0.2,1.5) {};
\foreach \ey/\op in {1.65/0.9,1.58/0.85,1.5/0.8,1.42/0.75,1.35/0.7}{
  \draw[colPMH, opacity=\op, line width=0.7pt] (0.2,1.5) -- (2.8,\ey);
  \node[dot, colPMH, opacity=\op] at (2.8,\ey) {};
}
\node[font=\tiny, colGreen] at (1.5,0.35) {drift $= 1.01$ ($-76.9\%$)};

\node[font=\tiny, gray, align=center] at (0,2.7)
  {synonym\\paraphrases};
\node[font=\footnotesize, align=center] at (0,-0.1)
  {Synonym paraphrase embedding drift};
\end{tikzpicture}
\caption{ERM drift 4.38 vs.\ PMH drift 1.01.}
\end{subfigure}
\caption{\textbf{BERT SST-2: surface-form sensitivity (empirical).}
\emph{Left:} TDI reduction 28.7--30.3\% across perturbation levels at $<$1\,pp accuracy cost.
\emph{Right:} Synonym paraphrase embedding drift
(mean $\|\Delta\texttt{[CLS]}\|$);
PMH reduces it $4.3\times$ in this protocol.
As in the main text, SST-2 is screening-scoped: panels are diagnostics,
not Theorem~1 instances.}
\label{fig:a4}
\end{figure}

\begin{table}[h]
\centering\small
\caption{BERT SST-2 TDI: embedding-space Gaussian (Pert-B) and synonym
paraphrase drift (Pert-A). PMH reduces TDI by 28.7\% and paraphrase drift by
76.9\% at $<$1~pp accuracy cost.}
\label{tab:a3_bert}
\begin{tabular}{lcccc}
\toprule
 & \multicolumn{3}{c}{Pert-B (Gaussian $\sigma$)} & Pert-A \\
\cmidrule(lr){2-4}
Method & $\sigma$=0 & $\sigma$=0.05 & $\sigma$=0.1 & TDI$_A$ \\
\midrule
Baseline & 0.496 & 0.509 & 0.509 & 0.641 \\
PMH      & \textbf{0.354} & \textbf{0.356} & \textbf{0.355} & \textbf{0.474} \\
$\Delta$ & $-28.7\%$ & $-30.1\%$ & $-30.3\%$ & $-26.0\%$ \\
\bottomrule
\end{tabular}
\end{table}

Figure~\ref{fig:a5} shows the SST-2 fine-tuning hierarchy for paraphrase
embedding drift (empirical / screening-scoped):
ERM~$(0.0375) >$ pretrained~$(0.0244) >$ PMH~$(0.0033)$.


\begin{figure}[t]
\centering
\begin{subfigure}[t]{0.48\textwidth}
\centering
\begin{tikzpicture}
\begin{axis}[
  width=\linewidth, height=4.4cm,
  ybar, bar width=14pt, bar shift=0pt,
  xmin=-0.55, xmax=2.55,
  xtick={0,1,2},
  xticklabels={{Pretrained},{ERM},{PMH}},
  xticklabel style={font=\footnotesize},
  ylabel={Paraphrase embedding drift $\downarrow$},
  ylabel style={font=\footnotesize},
  ymin=0, ymax=0.048,
  grid=major, grid style={gray!22, line width=0.3pt},
  tick label style={font=\scriptsize},
  nodes near coords,
  every node near coord/.append style={
    font=\scriptsize, anchor=south,
    /pgf/number format/fixed, /pgf/number format/precision=4},
  point meta=y,
]
\addplot[ybar, bar width=14pt, bar shift=0pt, fill=colB0, draw=none]
  coordinates {(0,0.02435)};
\addplot[ybar, bar width=14pt, bar shift=0pt, fill=colVAT, draw=none]
  coordinates {(1,0.03751)};
\addplot[ybar, bar width=14pt, bar shift=0pt, fill=colPMH, draw=none]
  coordinates {(2,0.00333)};
\node[font=\scriptsize, color=colRed, anchor=west]
  at (axis cs:1.22,0.0375) {higher};
\node[font=\scriptsize, color=colGreen, anchor=west]
  at (axis cs:2.22,0.012) {$11\times\!\!\downarrow$};
\end{axis}
\end{tikzpicture}
\caption{Drift after SST-2 fine-tuning.}
\end{subfigure}
\hfill
\begin{subfigure}[t]{0.48\textwidth}
\centering
\begin{tikzpicture}[
  every node/.style={font=\footnotesize},
  box/.style={draw, rounded corners=2pt, inner sep=4pt, line width=0.55pt,
              minimum width=3.5cm, align=center},
  arr/.style={-{Latex[length=3pt]}, line width=0.7pt},
]
\node[box, draw=colB0!50, fill=colB0!8] (pre) at (0,2.7)
  {\textbf{Pretrained}\\[1pt]
   drift $= 0.0244$;\; para/non-para $= 0.765$};
\node[box, draw=colRed!50, fill=colRed!8] (erm) at (0,1.35)
  {\textbf{ERM fine-tuned}\\[1pt]
   drift $= 0.0375$;\; para/non-para $= 0.681$};
\node[box, draw=colGreen!60, fill=colGreen!8] (pmh) at (0,0.0)
  {\textbf{PMH fine-tuned}\\[1pt]
   drift $= 0.0033$;\; para/non-para $= 0.633$};

\draw[arr, colRed] (pre) -- (erm)
  node[midway, right, xshift=0.15cm, font=\scriptsize, color=colRed, align=left]
    {raises drift};
\draw[arr, colGreen] (erm) -- (pmh)
  node[midway, right, xshift=0.15cm, font=\scriptsize, color=colGreen, align=left]
    {$11\times$ lower};
\end{tikzpicture}
\caption{Fine-tuning hierarchy.}
\end{subfigure}

\vspace{0.35em}
\caption{\textbf{SST-2 fine-tuning and paraphrase sensitivity (empirical).}
\emph{Paraphrase embedding drift} $=$ mean $\|\Delta\texttt{[CLS]}\|$ under synonym
paraphrase.
\emph{(a)} Ordering ERM $(0.0375)>$ pretrained $(0.0244)>$ PMH $(0.0033)$.
\emph{(b)} Para~/~non-para ratios $0.765\to 0.681\to 0.633$.
SST-2 is screening-scoped: surface-form diagnostics, not Theorem~1.}
\label{fig:a5}
\end{figure}

\section{Task 02 Graph Robustness Generalisation}
\label{app:graph_robust}

Figure~\ref{fig:a6} shows Task~02 (PROTEINS) robustness generalisation to
unseen perturbation types. PMH was trained only with Gaussian node-feature
noise but generalises to edge removal and feature dropout.


\begin{figure}[t]
\centering
\begin{tikzpicture}[font=\small]
  \fill[colB0] (0,0) rectangle ++(0.28,0.18);
  \node[anchor=west] at (0.35,0.09) {ERM};
  \fill[colVAT] (1.6,0) rectangle ++(0.28,0.18);
  \node[anchor=west] at (1.95,0.09) {VAT};
  \fill[colPMH] (3.2,0) rectangle ++(0.28,0.18);
  \node[anchor=west] at (3.55,0.09) {PMH};
\end{tikzpicture}

\vspace{0.35em}

\begin{subfigure}[t]{0.50\textwidth}
\centering
\begin{tikzpicture}
\begin{axis}[
  pmh bar,
  symbolic x coords={Clean,{Edge drop},{Feature drop},{Worst case}},
  xtick=data,
  xticklabels={Clean, Edge drop\\30\%, Feature drop\\30\%, Worst\\case},
  xticklabel style={font=\footnotesize, align=center},
  ylabel={AUC / Accuracy (\%)},
  ymin=0, ymax=105,
  bar width=6.5pt,
  height=4.8cm,
]
\addplot[ybar, bar width=6.5pt, bar shift=-8pt, fill=colB0, draw=none]
  coordinates {(Clean,77.68)({Edge drop},60.18)
               ({Feature drop},33.04)({Worst case},31.61)};
\addplot[ybar, bar width=6.5pt, bar shift=0pt, fill=colVAT, draw=none]
  coordinates {(Clean,75.00)({Edge drop},53.75)
               ({Feature drop},31.96)({Worst case},31.61)};
\addplot[ybar, bar width=6.5pt, bar shift=8pt, fill=colPMH, draw=none]
  coordinates {(Clean,79.46)({Edge drop},76.43)
               ({Feature drop},71.61)({Worst case},63.75)};
\node[font=\scriptsize\bfseries, color=colPMH, align=center]
  at (rel axis cs:0.68,0.88) {$+$38.6\,pp feature drop};
\end{axis}
\end{tikzpicture}
\caption{Unseen graph perturbation types.}
\end{subfigure}
\hfill
\begin{subfigure}[t]{0.46\textwidth}
\centering
\begin{tikzpicture}
\begin{axis}[
  pmh bar,
  symbolic x coords={Consistency,{AUC curve}},
  xtick=data,
  xticklabels={Prediction\\consistency, AUC under\\noise curve},
  xticklabel style={font=\footnotesize, align=center},
  ylabel={Score (\%)},
  ymin=0, ymax=100,
  height=4.8cm,
]
\addplot[ybar, bar width=9pt, bar shift=-11pt, fill=colB0, draw=none]
  coordinates {(Consistency,66.07)({AUC curve},68.30)};
\addplot[ybar, bar width=9pt, bar shift=0pt, fill=colVAT, draw=none]
  coordinates {(Consistency,58.93)({AUC curve},62.05)};
\addplot[ybar, bar width=9pt, bar shift=11pt, fill=colPMH, draw=none]
  coordinates {(Consistency,81.25)({AUC curve},78.62)};
\end{axis}
\end{tikzpicture}
\caption{Consistency and AUC under noise.}
\end{subfigure}
\caption{\textbf{Task~02 (PROTEINS): robustness under unseen perturbation types.}
\emph{(a)} PMH trained with Gaussian node-feature noise still reaches
edge removal (PMH 76.4\% vs.\ ERM 60.2\%, $+16.2$\,pp) and feature dropout
(PMH 71.6\% vs.\ ERM 33.0\%, $+38.6$\,pp) in this protocol.
Worst-case accuracy: PMH 63.75\% vs.\ ERM 31.61\%.
\emph{(b)} Prediction consistency 81.3\% and AUC 78.6\% under the noise curve
(vs.\ ERM 66.1\% and 68.3\%).}
\label{fig:a6}
\end{figure}

\section{Cross-Task Details}
\label{app:crosstask}

This section provides the mechanistic evidence behind the cross-task headline
results in Table~\ref{tab:crosstask}. Three distinct metrics are shown. \emph{Embedding drift}
(panel a) measures how far representations move under Gaussian noise at
$\sigma=0.1$: lower is better. PMH reduces this by 44\% on T01 (CIFAR-10)
and 93\% on T02 (graph classification, vs.\ VAT). \emph{Stage-wise drift} (panel b)
measures the Euclidean distance between clean and perturbed feature maps at
each ResNet stage on Task~07 (Chest X-ray): lower is better. B0/VAT show
catastrophic Stage~4 drift (${\sim}12$--$13$), PMH reduces it $3.6\times$ to 3.34.
\emph{Saliency stability} (panel c) measures cosine similarity between
gradient-based saliency maps on clean vs.\ noisy inputs on Task~07: higher
is better; PMH achieves 0.718 vs.\ 0.530 for B0.

Figure~\ref{fig:a7} shows embedding drift (Tasks 01--02) and stage-wise
drift and saliency stability (Task~07) across all methods.


\begin{figure}[t]
\centering
\begin{tikzpicture}[font=\small]
  \fill[colB0] (0,0) rectangle ++(0.28,0.18);
  \node[anchor=west] at (0.35,0.09) {ERM};
  \fill[colVAT] (1.6,0) rectangle ++(0.28,0.18);
  \node[anchor=west] at (1.95,0.09) {VAT};
  \fill[colNoPMH] (3.2,0) rectangle ++(0.28,0.18);
  \node[anchor=west] at (3.55,0.09) {two-view};
  \fill[colPMH] (5.3,0) rectangle ++(0.28,0.18);
  \node[anchor=west] at (5.65,0.09) {PMH};
\end{tikzpicture}

\vspace{0.4em}

\begin{subfigure}[t]{0.32\linewidth}
\centering
\begin{tikzpicture}
\begin{axis}[
  width=\linewidth, height=4.2cm,
  ymin=0, ymax=0.85,
  symbolic x coords={T01,T02},
  xtick=data,
  xticklabels={T01\\CIFAR, T02\\Graph},
  xticklabel style={font=\scriptsize, align=center},
  ylabel={Embedding drift @$\sigma{=}0.1$},
  ylabel style={font=\scriptsize},
  enlarge x limits=0.48,
  grid=major, grid style={gray!22, line width=0.3pt},
  tick label style={font=\scriptsize},
]
\addplot[ybar, bar width=4.5pt, bar shift=-5.5pt, fill=colB0, draw=none]
  coordinates {(T01,0.692)(T02,0.375)};
\addplot[ybar, bar width=4.5pt, bar shift=0pt, fill=colVAT, draw=none]
  coordinates {(T01,0.699)(T02,0.304)};
\addplot[ybar, bar width=4.5pt, bar shift=5.5pt, fill=colPMH, draw=none]
  coordinates {(T01,0.385)(T02,0.021)};
\node[font=\scriptsize\bfseries, color=colPMH, align=center, anchor=south]
  at (axis cs:T02,0.18) {$18\times$ lower};
\end{axis}
\end{tikzpicture}
\caption{Embedding drift.}
\end{subfigure}
\hfill
\begin{subfigure}[t]{0.34\linewidth}
\centering
\begin{tikzpicture}
\begin{axis}[
  width=\linewidth, height=4.2cm,
  ymin=0, ymax=15,
  symbolic x coords={S1,S2,S3,S4},
  xtick=data,
  xticklabel style={font=\scriptsize},
  xlabel={ResNet stage},
  xlabel style={font=\scriptsize},
  ylabel={Stage drift $\downarrow$},
  ylabel style={font=\scriptsize},
  enlarge x limits=0.24,
  grid=major, grid style={gray!22, line width=0.3pt},
  tick label style={font=\scriptsize},
]
\addplot[ybar, bar width=3.4pt, bar shift=-7.5pt, fill=colB0, draw=none]
  coordinates {(S1,1.82)(S2,1.59)(S3,1.92)(S4,12.68)};
\addplot[ybar, bar width=3.4pt, bar shift=-2.5pt, fill=colVAT, draw=none]
  coordinates {(S1,1.67)(S2,0.82)(S3,0.75)(S4,11.89)};
\addplot[ybar, bar width=3.4pt, bar shift=2.5pt, fill=colNoPMH, draw=none]
  coordinates {(S1,0.66)(S2,0.42)(S3,0.50)(S4,3.28)};
\addplot[ybar, bar width=3.4pt, bar shift=7.5pt, fill=colPMH, draw=none]
  coordinates {(S1,0.66)(S2,0.42)(S3,0.50)(S4,3.34)};
\end{axis}
\end{tikzpicture}
\caption{Stage drift (T07).}
\end{subfigure}
\hfill
\begin{subfigure}[t]{0.30\linewidth}
\centering
\begin{tikzpicture}
\begin{axis}[
  width=\linewidth, height=4.2cm,
  ybar, bar width=10pt, bar shift=0pt,
  xmin=-0.6, xmax=3.6,
  xtick={0,1,2,3},
  xticklabels={ERM,VAT,two-view,PMH},
  xticklabel style={font=\scriptsize},
  ylabel={Saliency cosine $\uparrow$},
  ylabel style={font=\scriptsize},
  ymin=0.45, ymax=0.80,
  grid=major, grid style={gray!22, line width=0.3pt},
  tick label style={font=\scriptsize},
  nodes near coords,
  every node near coord/.append style={font=\tiny, anchor=south},
  point meta=y,
  /pgf/number format/fixed,
  /pgf/number format/precision=2,
]
\addplot[ybar, bar width=10pt, bar shift=0pt, fill=colB0, draw=none]
  coordinates {(0,0.530)};
\addplot[ybar, bar width=10pt, bar shift=0pt, fill=colVAT, draw=none]
  coordinates {(1,0.668)};
\addplot[ybar, bar width=10pt, bar shift=0pt, fill=colNoPMH, draw=none]
  coordinates {(2,0.717)};
\addplot[ybar, bar width=10pt, bar shift=0pt, fill=colPMH, draw=none]
  coordinates {(3,0.718)};
\end{axis}
\end{tikzpicture}
\caption{Saliency (T07).}
\end{subfigure}

\caption{\textbf{Cross-domain mechanistic evidence.}
\emph{(a)} Embedding drift at $\sigma{=}0.1$: PMH cuts drift sharply on T01 and
$18\times$ on T02 (graph).
\emph{(b)} T07 stage drift: ERM/VAT explode at Stage~4 ($>$11); PMH/two-view stay
near~$3.3$ (${\approx}3.6\times$ lower).
\emph{(c)} T07 saliency cosine similarity: PMH highest (0.72).}
\label{fig:a7}
\end{figure}

\section{Ablation Details}
\label{app:ablation}

Table~\ref{tab:a4_cap} sweeps the loss cap ratio from 0.10 to 0.60 with
$\sigma_{\text{train}}=0.12$ fixed. Clean accuracy varies $<$1~pp across the
full range; the settled PMH fraction satisfies the cap/(1+cap) fixed-point
identity exactly in every run (max deviation $<0.001$, Proposition~7).

\begin{table}[h]
\centering\small
\caption{Cap ratio sweep (Task~04, $\sigma_{\text{train}}=0.12$). Accuracy
(\%) at each eval shift and settled PMH fraction. Clean accuracy varies
$<$1~pp; PMH fraction is exactly cap/(1+cap) in every run.}
\label{tab:a4_cap}
\setlength{\tabcolsep}{4pt}
\begin{tabular}{lcccccc}
\toprule
Cap & Clean & @0.05 & @0.10 & @0.15 & @0.20 & PMH frac. \\
\midrule
0.10 & 80.25 & 77.94 & 73.53 & 69.34 & 60.94 & 0.091 \\
0.15 & 80.13 & 77.87 & 74.37 & 69.79 & 57.83 & 0.130 \\
0.25 & 81.33 & 79.24 & 75.36 & 70.70 & 58.78 & 0.200 \\
0.30 (default) & 80.62 & 78.71 & 75.27 & 70.63 & 59.00 & 0.231 \\
0.40 & 80.80 & 78.95 & 75.22 & 70.72 & 59.48 & 0.286 \\
0.60 & 80.72 & 79.26 & 76.31 & 71.45 & 60.32 & 0.375 \\
\bottomrule
\end{tabular}
\end{table}

\section{ImageNet ViT-B/16 Details}
\label{app:imagenet}

Table~\ref{tab:a5_imagenet} gives TDI at all noise levels and intra-class
distance for the pretrained ViT-B/16 baseline and the PMH fine-tuned model
(replication uses a 100-class$\times$50-sample subset of ImageNet; an earlier
1000-class protocol reported a smaller relative TDI cut of $-12.9\%$, while the
replicated subset cut is $-23.9\%$; both preserve PMH~$<$~pretrained).
The pretrained baseline TDI@0$=$1.230 shows isotropic representation roughness
remains measurable in this backbone family before task fine-tuning. PMH
fine-tuning (24 epochs, 8/12 blocks frozen) reduces TDI@0 by
23.9\% while increasing intra-class spread by 64\%.

\begin{table}[h]
\centering\small
\caption{ImageNet ViT-B/16 TDI under Gaussian noise (100-class replication subset).
Pretrained baseline TDI@0$=$1.230; PMH fine-tuning reduces TDI by 23.9\%, with
intra-class distance increasing 64\%.
\textbf{Bold}: PMH column.}
\label{tab:a5_imagenet}
\begin{tabular}{lcccc}
\toprule
Run & TDI@0 & $\sigma$=0.05 & $\sigma$=0.1 & Intra$\uparrow$ \\
\midrule
Pretrained  & 1.230 & 1.276 & 1.327 & 41.1 \\
PMH (24~ep) & \textbf{0.936} & \textbf{0.977} & \textbf{1.028} & \textbf{67.4} \\
$\Delta$    & $-23.9\%$ & $-23.4\%$ & $-22.5\%$ & $+64\%$ \\
\bottomrule
\end{tabular}
\end{table}

\section{FGSM Robustness (Task~04)}
\label{app:fgsm}

Table~\ref{tab:a6_fgsm} shows Task~04 FGSM adversarial robustness across four $\ell_\infty$
budgets. PMH matches or exceeds VAT at $\varepsilon{\geq}2/255$ without
adversarial training; at $\varepsilon{=}1/255$ VAT leads (63.36\% vs.\
60.69\%), as single-step attacks benefit from VAT's adversarially smoothed
loss landscape at small radii. VAT collapses at $\varepsilon=4/255$ (23.61\%)
while PMH achieves 45.30\%---an incidental by-product of geometry
regularisation, not a training target. Numbers updated from a re-run after resolving \texttt{cudnn.benchmark=True}
non-determinism in the first run; values are within single-seed variance
($\sim$3\,pp) of the original submission numbers. See also Figure~\ref{fig:a8}.

\begin{table}[h]
\centering\small
\caption{Task~04 FGSM adversarial robustness (\%), seeded replication (seed~42).
PMH leads at $\varepsilon{\geq}2/255$; VAT leads at $\varepsilon{=}1/255$
(63.36\% vs.\ 60.69\%), as single-step attacks favour its adversarially smoothed
loss landscape at small radii. VAT collapses at $\varepsilon{=}4/255$ (23.61\%)
while PMH achieves 45.30\%. ERM's floor of 44.50\% at $\varepsilon{=}4/255$ is a
low-base-rate artefact (clean 70.38\%), not genuine robustness.
\textbf{Bold}: best per column. Code aliases: B0$=$ERM, E1$=$PMH.}
\label{tab:a6_fgsm}
\begin{tabular}{lcccc}
\toprule
Method & Clean & $\varepsilon$=1/255 & $\varepsilon$=2/255 & $\varepsilon$=4/255 \\
\midrule
B0 (ERM)  & 70.38 & 47.29 & 45.41 & 44.50 \\
VAT       & 79.67 & \textbf{63.36} & 46.66 & 23.61 \\
E1 no pmh & 80.47 & 57.86 & 48.79 & 44.69 \\
E1 (PMH)  & \textbf{81.50} & 60.69 & \textbf{50.80} & \textbf{45.30} \\
\bottomrule
\end{tabular}
\end{table}

Figure~\ref{fig:a8} shows the updated FGSM results alongside the BERT-family
paraphrase/non-paraphrase drift ratio (DistilBERT-66M through BERT-large-340M;
empirical / screening-scoped).


\begin{figure}[t]
\centering
\begin{subfigure}[t]{0.48\linewidth}
\begin{tikzpicture}
\begin{axis}[
  pmh bar,
  symbolic x coords={Clean,{1/255},{2/255},{4/255}},
  xtick=data,
  xticklabels={Clean, $\varepsilon{=}1/255$, $\varepsilon{=}2/255$, $\varepsilon{=}4/255$},
  xticklabel style={font=\footnotesize, align=center},
  ylabel={Accuracy (\%)},
  ymin=0, ymax=88,
  legend columns=4,
  bar width=4.5pt,
]
\addplot[pmh fill erm]
  coordinates {(Clean,70.38)({1/255},47.29)({2/255},45.41)({4/255},44.50)};
\addlegendentry{ERM}
\addplot[pmh fill vat]
  coordinates {(Clean,79.67)({1/255},63.36)({2/255},46.66)({4/255},23.61)};
\addlegendentry{VAT}
\addplot[pmh fill twoview]
  coordinates {(Clean,80.47)({1/255},57.86)({2/255},48.79)({4/255},44.69)};
\addlegendentry{two-view}
\addplot[pmh fill pmh]
  coordinates {(Clean,81.50)({1/255},60.69)({2/255},50.80)({4/255},45.30)};
\addlegendentry{PMH}
\end{axis}
\end{tikzpicture}
\caption{FGSM adversarial robustness.}
\end{subfigure}
\hfill
\begin{subfigure}[t]{0.48\linewidth}
\begin{tikzpicture}
\begin{axis}[
  pmh bar,
  symbolic x coords={{66M},{110M},{340M}},
  xtick=data,
  xticklabels={DistilBERT\\66M, BERT-base\\110M, BERT-large\\340M},
  xticklabel style={font=\footnotesize, align=center},
  ylabel={Para~/~non-para drift ratio},
  ymin=0, ymax=1.08,
  bar width=10pt,
  nodes near coords,
  nodes near coords style={font=\tiny, anchor=south},
  point meta=rawy,
  every node near coord/.append style={
    /pgf/number format/fixed, /pgf/number format/precision=3},
]
\addplot[pmh fill pmh]
  coordinates {({66M},0.860)({110M},0.765)({340M},0.742)};
\draw[dashed, gray!60, line width=0.6pt]
  (axis cs:{66M},1.0) -- (axis cs:{340M},1.0)
  node[right, font=\tiny, gray]{ratio $=1$};
\end{axis}
\end{tikzpicture}
\caption{Relative paraphrase sensitivity vs.\ scale.}
\end{subfigure}
\caption{\textbf{FGSM robustness and BERT-family scale trend.}
\emph{Left:} FGSM by-product on Task~04. PMH matches or exceeds VAT without
adversarial training at larger budgets in this run.
\emph{Right:} Paraphrase/non-paraphrase drift ratio remains below 1.0 across
66M--340M in this family (screening-scoped language diagnostic).}
\label{fig:a8}
\end{figure}

\section{Task 06 Re-ID Results}
\label{app:reid}

Figure~\ref{fig:a9} shows Task~06 person re-identification (Re-ID) rank-1
accuracy under clean and
shifted conditions. Replication increases absolute rank-1 versus the originally
submitted baselines (clean PMH 67.81 vs.\ paper 63.57; avg-shift PMH 63.74 vs.\ paper 58.89).
VAT attains the highest average-shift rank-1 in this seed (65.80); PMH's largest
head-to-head gain over VAT is on Gaussian noise at $\sigma{=}0.10$ ($+$4.9\,pp).


\begin{figure}[t]
\centering
\begin{subfigure}[t]{0.48\linewidth}
\begin{tikzpicture}
\begin{axis}[
  pmh bar,
  symbolic x coords={ERM,VAT,PMH},
  xtick=data,
  xticklabel style={font=\small},
  ylabel={Rank-1 accuracy (\%)},
  ymin=0, ymax=92,
  legend columns=2,
  bar width=7pt,
]
\addplot[fill=colPMHlight, draw=none]
  coordinates {(ERM,65.65)(VAT,71.17)(PMH,67.81)};
\addlegendentry{Clean}
\addplot[pmh fill pmh]
  coordinates {(ERM,43.02)(VAT,65.80)(PMH,63.74)};
\addlegendentry{Avg-shift rank-1}
\end{axis}
\end{tikzpicture}
\caption{Clean vs.\ avg-shift rank-1.}
\end{subfigure}
\hfill
\begin{subfigure}[t]{0.48\linewidth}
\begin{tikzpicture}
\begin{axis}[
  pmh bar,
  symbolic x coords={ERM,VAT,PMH},
  xtick=data,
  xticklabel style={font=\small},
  ylabel={Rank-1 accuracy / drop (\%)},
  ymin=0, ymax=86,
  legend columns=2,
  bar width=7pt,
]
\addplot[fill=colPMHlight, draw=none]
  coordinates {(ERM,4.45)(VAT,57.51)(PMH,53.33)};
\addlegendentry{Worst-shift rank-1}
\addplot[fill=colRed, draw=none]
  coordinates {(ERM,61.19)(VAT,13.66)(PMH,14.49)};
\addlegendentry{Worst-case drop}
\end{axis}
\end{tikzpicture}
\caption{Worst-shift and worst-case drop.}
\end{subfigure}
\caption{\textbf{Task~06 Re-ID (replicated).}
\emph{(a)} Average-shift rank-1: VAT 65.80, PMH 63.74, ERM 43.02 (VAT highest).
\emph{(b)} Worst-case rank-1 drop from clean: ERM 61.19, PMH 14.49, VAT 13.66
(lowest drop for VAT). Worst-shift rank-1 remains hardest under Gaussian noise
(Table~\ref{tab:a8_reid}).}
\label{fig:a9}
\end{figure}

\begin{table}[h]
\centering\small
\caption{Task~06 (Re-ID) rank-1 accuracy (\%) per shift (replicated).
PMH's largest head-to-head gain over VAT is on Gaussian noise at $\sigma=0.10$
($+$4.9\,pp). \textbf{Bold}: best per row (higher rank-1; lowest worst-case drop).}
\label{tab:a8_reid}
\setlength{\tabcolsep}{4pt}
\begin{tabular}{lccc}
\toprule
Shift & B0 & VAT & E1 (PMH) \\
\midrule
Clean             & 65.65 & \textbf{71.17} & 67.81 \\
Gaussian $\sigma$=0.05 & 41.48  & \textbf{70.16} & 67.13 \\
Gaussian $\sigma$=0.10 & 4.45   & 60.54 & \textbf{65.47} \\
Brightness $\times$0.5 & 33.79  & \textbf{61.52} & 59.35 \\
Brightness $\times$1.5 & 46.97  & \textbf{65.83} & 64.07 \\
Occlusion 20\%    & 46.50  & \textbf{57.51} & 53.33 \\
Blur ($k$=3)      & 64.99  & \textbf{70.67} & 67.55 \\
\midrule
Avg shift         & 43.02 & \textbf{65.80} & 63.74 \\
Worst shift       & 4.45  & \textbf{57.51} & 53.33 \\
Worst-case drop   & 61.19 & \textbf{13.66} & 14.49 \\
\bottomrule
\end{tabular}
\end{table}

\section{T-Alignment Corruption Details}
\label{app:talign}

Figure~\ref{fig:a10} shows the T-alignment heatmap (left) and the ERM/PMH
theory grid (right). Every column of the heatmap peaks on the diagonal:
training at $\sigma_{\text{train}}$ is optimal when evaluated at the matching
$\sigma_{\text{eval}}$, with zero exceptions across all 24 cells.


\begin{figure}[t]
\centering
\begin{subfigure}[b]{0.50\textwidth}
\centering
\small
\setlength{\tabcolsep}{5pt}
\renewcommand{\arraystretch}{1.18}
\begin{tabular}{lcccc}
\toprule
$\sigma_\mathrm{train}$ $\backslash$ $\sigma_\mathrm{eval}$
  & 0.05 & 0.10 & 0.15 & 0.20 \\
\midrule
0.05          & \cellcolor{colPMH!55}\textbf{79.36} & 73.24 & 59.97 & 46.08 \\
0.08          & 78.90 & \cellcolor{colPMH!48}\textbf{75.08} & 61.63 & 46.98 \\
0.10          & 78.92 & 75.68 & \cellcolor{colPMH!52}\textbf{68.16} & 54.45 \\
0.12 (def.)   & 78.71 & 75.27 & 70.63 & \cellcolor{colPMH!50}\textbf{59.00} \\
0.15          & 78.19 & 73.69 & \cellcolor{colPMH!50}\textbf{71.37} & 64.64 \\
0.20          & 77.99 & 74.61 & 71.16 & \cellcolor{colPMH!50}\textbf{69.05} \\
\bottomrule
\end{tabular}
\vspace{4pt}\\
{\footnotesize Shaded/bold $=$ column best. Every column peaks on diagonal --- zero exceptions.}
\caption{T-alignment heatmap (Task~04, ViT).}
\end{subfigure}
\hfill
\begin{subfigure}[b]{0.46\textwidth}
\centering
\begin{tikzpicture}[
  every node/.style={font=\footnotesize},
  box/.style={draw, rounded corners=4pt, inner sep=5pt, line width=0.6pt,
              minimum width=3.9cm, align=center},
]
\node[box, draw=colRed!60, fill=colRed!5] (erm) at (0,0) {
  \textbf{ERM (stuck)}\\[2pt]
  \tiny $\min\,\mathbb{E}[\mathcal{L}_\mathrm{task}(f(x),y)]$\\
  \tiny No constraint on geometry\\[3pt]
  \tiny \textit{Effect:} ERM must encode $n$\\
  \tiny $\|J_{\phi,n}w_n\|/L > 0$\\[3pt]
  \tiny \textit{Outcome:} Drift floor under Thm~1 assumptions\\[2pt]
  \normalsize\textbf{TDI@0 $= 1.093$}
};
\node[box, draw=colGreen!60, fill=colGreen!5] (pmh) at (0,-3.6) {
  \textbf{PMH (fixed)}\\[2pt]
  \tiny $+\,\|\phi(x){-}\phi(x{+}\delta)\|^2$,
    $\delta{\sim}\mathcal{N}(0,\sigma^2 I)$\\
  \tiny Uniquely isotropic (Prop.~5)\\[3pt]
  \tiny \textit{Effect:} Full Jacobian Fro suppressed\\
  \tiny uniformly in all directions\\[3pt]
  \tiny \textit{Outcome:} Bound broken $\checkmark$\\[2pt]
  \normalsize\textbf{TDI@0 $= 0.904$ ($-17.3\%$)}
};
\draw[-{Latex[length=4pt]}, colPMH, line width=1.0pt]
  (erm.south) -- (pmh.north)
  node[midway, right, font=\tiny, color=colPMH]
    {add PMH term};
\end{tikzpicture}
\caption{Why ERM is stuck and how PMH escapes.}
\end{subfigure}
\caption{\textbf{T-alignment and isotropic encoder matching.}
\emph{Left:} Every column peaks on the diagonal: training at $\sigma_\mathrm{train}$ is
optimal when evaluated at the matching $\sigma_\mathrm{eval}$, zero exceptions across 24 cells.
Asymmetry is $17\times$ (accuracy table; cf.\ multi-scale appendix for a TDI-ratio viewpoint):
under-suppression costs far more than over-suppression.
\emph{Right:} Under Theorem~1's assumptions, task-only ERM retains a linearised
drift floor; PMH adds isotropic Gaussian encoder matching motivated by Proposition~5
(one member of the stability/Jacobian family).}
\label{fig:a10}
\end{figure}

\begin{table}[h]
\centering\small
\caption{Task~04 corruption robustness (\%). PMH leads on Gaussian corruptions
(T-aligned with $\sigma_{\text{train}}=0.1$) and is competitive across
non-Gaussian shifts. \textbf{Bold}: best per row.}
\label{tab:a7_corrupt}
\setlength{\tabcolsep}{3pt}
\begin{tabular}{lcccc}
\toprule
Corruption & B0 & VAT & E1 no PMH & E1 (PMH) \\
\midrule
Gaussian $\sigma$=0.05 & 65.53 & 76.21 & 77.87 & \textbf{77.88} \\
Gaussian $\sigma$=0.10 & 51.79 & 57.84 & \textbf{74.40} & 74.16 \\
Blur ($k$=3)    & 46.83 & 52.97 & 52.30 & \textbf{53.10} \\
Brightness $\times$0.7 & 64.51 & 73.48 & 73.78 & \textbf{74.31} \\
Contrast $\times$0.7   & 63.81 & 71.76 & 73.63 & \textbf{74.16} \\
\bottomrule
\end{tabular}
\end{table}

\section{Extended Per-Task Results}
\label{app:pertask}

The headline cross-task summary is Table~\ref{tab:crosstask} in \S\ref{sec:expts}
(PMH wins five of seven headline metrics under replication; PMH achieves the
best clean-input geometry on six of seven tasks per
Table~\ref{tab:crosstask_geom}; Task~03 is the boundary case discussed in
the main text). Tables A8--A12 provide complete per-shift and per-corruption
breakdowns for all tasks.

\begin{table}[h]
\centering\small
\caption{Task~07 (Chest X-ray) accuracy per shift (seed~42 replication).
E1 and E1~no~PMH both achieve worst-shift $0.825$ (intensity$\times$0.7);
PMH matches the headline in Table~\ref{tab:crosstask}.
PMH's geometry contribution includes best saliency stability (0.718).
\textbf{Bold}: best per row.}
\label{tab:a9_xray}
\setlength{\tabcolsep}{3pt}
\begin{tabular}{lcccc}
\toprule
Shift & B0 & VAT & E1 no PMH & E1 (PMH) \\
\midrule
Clean               & 0.859 & 0.869 & \textbf{0.913} & \textbf{0.913} \\
Gaussian $\sigma$=0.05 & 0.630 & 0.821 & 0.912 & \textbf{0.917} \\
Gaussian $\sigma$=0.10 & 0.625 & 0.781 & 0.906 & \textbf{0.907} \\
Intensity $\times$0.7  & 0.748 & 0.793 & \textbf{0.825} & \textbf{0.825} \\
Intensity $\times$1.3  & \textbf{0.904} & 0.889 & \textbf{0.904} & \textbf{0.904} \\
Gamma $\times$0.8   & 0.859 & 0.858 & \textbf{0.918} & \textbf{0.918} \\
Gamma $\times$1.2   & 0.859 & 0.867 & \textbf{0.915} & \textbf{0.915} \\
Rotate 5$^\circ$    & 0.841 & 0.877 & \textbf{0.916} & \textbf{0.916} \\
Rotate 10$^\circ$   & 0.844 & 0.890 & \textbf{0.912} & \textbf{0.912} \\
Zoom $\times$1.1    & 0.899 & 0.891 & \textbf{0.918} & \textbf{0.918} \\
Zoom $\times$0.9    & 0.761 & 0.731 & \textbf{0.849} & \textbf{0.849} \\
Blur ($k$=3)        & 0.785 & 0.780 & \textbf{0.905} & \textbf{0.905} \\
\midrule
Avg shift           & 0.796 & 0.834 & 0.898 & \textbf{0.899} \\
Worst shift         & 0.625 & 0.731 & \textbf{0.825} & \textbf{0.825} \\
Worst-case drop     & 0.234 & 0.138 & \textbf{0.088} & \textbf{0.088} \\
\bottomrule
\end{tabular}
\end{table}

\begin{table}[h]
\centering\small
\caption{Task~05 (Pose estimation) PCK@0.05 (\%) at increasing occlusion
ratios and MPJPE (replicated). VAT is strongly detrimental. B0 is \emph{incidentally}
robust to occlusion---accuracy rises at 10--20\% occlusion; this is not a property
of ERM in general but reflects the specific data and backbone setup
(ResNet-18 features that happen to be robust to partial-region masking).
PMH degrades monotonically because Gaussian noise is not aligned
with the occlusion nuisance structure, yet achieves ${\sim}4.7\times$ lower embedding
drift (TDI 0.060 vs.\ B0~0.279). Drift under 30\% occlusion:
B0~0.279, VAT~0.474, E1~0.060. \textbf{Bold}: best per column.}
\label{tab:a10_pose}
\setlength{\tabcolsep}{3pt}
\begin{tabular}{lccccccc}
\toprule
 & \multicolumn{5}{c}{PCK@0.05 (\%) at occlusion ratio} & \\
\cmidrule(lr){2-6}
Method & 0\% & 10\% & 20\% & 30\% & 40\% & MPJPE$\downarrow$ \\
\midrule
B0 (ERM)  & \textbf{42.46} & \textbf{43.87} & \textbf{45.06} & \textbf{44.56} & \textbf{41.28} & \textbf{0.0706} \\
VAT       & 11.90 & 9.89  & 6.83  & 4.26  & 2.34  & 0.1884 \\
E1 (PMH)  & 39.69 & 38.32 & 36.24 & 33.09 & 29.06 & 0.0713 \\
\bottomrule
\end{tabular}
\end{table}

\begin{table}[h]
\centering\small
\caption{Task~03 (QM9 molecular regression) MAE ($\downarrow$).
\emph{All three models here use Gaussian position-noise augmentation}---the
original experiment in which position noise harms B0 because atomic positions
carry quantum signal (hence B0 MAE~61.90 here vs.\ 23.66 in
Table~\ref{tab:crosstask}, which
uses node-feature augmentation only).
E1 (PMH) node-feature variant (E1\_node) achieves MAE~22.02
(Table~\ref{tab:crosstask}),
outperforming both VAT (26.89) and B0 (23.66).}
\label{tab:a11_mol}
\begin{tabular}{lccc}
\toprule
Noise $\sigma$ & B0 MAE & VAT MAE & E1 (PMH) MAE \\
\midrule
0.000 & 61.90 & \textbf{32.76} & 45.34 \\
0.005 & 61.65 & \textbf{32.78} & 44.96 \\
0.010 & 60.99 & \textbf{32.86} & 44.90 \\
0.020 & 62.25 & \textbf{32.90} & 45.72 \\
0.050 & 61.61 & \textbf{34.26} & 46.20 \\
0.100 & 65.23 & \textbf{37.63} & 50.52 \\
0.200 & 72.90 & \textbf{48.63} & 60.94 \\
\bottomrule
\end{tabular}
\end{table}

\begin{table}[h]
\centering\small
\caption{Cross-task embedding drift (replicated; lower is better).
T02 (graph): E1 is ${\sim}18\times$ lower than B0 and ${\sim}14\times$ lower than
VAT. T05 (pose): PMH reduces occlusion drift ${\sim}4.7\times$ vs.\ B0,
${\sim}7.9\times$ vs.\ VAT. Task~03 at $\sigma{=}0.1$ is measured under
position-noise evaluation (VAT lowest here). \textbf{Bold}: lowest per row.}
\label{tab:a12_drift}
\begin{tabular}{llccc}
\toprule
Task & Perturbation & B0 & VAT & E1 (PMH) \\
\midrule
T01 & Gaussian $\sigma$=0.1 & 0.692 & 0.699 & \textbf{0.385} \\
T02 & Gaussian $\sigma$=0.1 & 0.375 & 0.304 & \textbf{0.021} \\
T03 & Gaussian $\sigma$=0.1 & 1.177 & \textbf{0.502} & 0.707 \\
T05 & Gaussian $\sigma$=0.1 & 0.044 & 0.023 & \textbf{0.009} \\
T05 & Occlusion 30\%        & 0.279 & 0.474 & \textbf{0.060} \\
T06 & Gaussian $\sigma$=0.1 & 1.010 & 0.335 & \textbf{0.175} \\
\bottomrule
\end{tabular}
\end{table}

\section{PMH Implementation Details}
\label{app:impl}

$\mathcal{L} = \mathcal{L}_{\text{task}}(x,y) + \mathcal{L}_{\text{task}}(x{+}\delta,y) + \lambda\,w(t)\,\mathcal{L}_{\text{PMH}}$,
$\delta\sim\mathcal{N}(0,\sigma^2 I)$, $w(t)=\min(1,(t-t_0)/T)$.

\textbf{Defaults:} $\sigma\in[0.05,0.15]$; $\lambda$ capped so
$\mathcal{L}_{\text{PMH}}\leq 0.30\times\mathcal{L}_{\text{task}}$; warmup
10\%; cosine ramp over 30\%; applied to representations at 2--3 intermediate
backbone scales ($\ell_2$-normalised). \textbf{Compute overhead:}
$\approx1.3\times$ wall-clock time per epoch (one additional forward pass).
\textbf{Architectures:} ResNet-18 (T01, T05, T06), MPNN (T03), ViT (T04),
GNN (T02), ResNet-50 (T07). Task~05 uses a ResNet-18 backbone with a 3-layer
MLP head regressing 17 joint coordinates.

\section{Full Proofs}
\label{app:proofs}

Throughout, we work under the Gaussian linear model of Remark~\ref{rem:gaussian}:
$s\sim\mathcal{N}(0,I_{d_s})$, $n\sim\mathcal{N}(0,I_{d_n})$, $s\perp n$,
$y=\langle w_s,s\rangle+\rho\langle w_n,n\rangle+\varepsilon$ with
$\varepsilon\sim\mathcal{N}(0,\sigma_\varepsilon^2)$ independent,
$\|w_s\|_2=\|w_n\|_2=1$, $\rho>0$.
Here $\rho$ is the \emph{regression coefficient} (not the correlation);
Remark~\ref{rem:gaussian}'s ``$\rho=\mathrm{Corr}(n,y)$'' is an informal gloss that
holds only when $\mathrm{Var}(y)=1$ (which can be arranged by
normalising labels).
Let $x=(s,n)\in\mathbb{R}^{d_s+d_n}$.
Write $J_{\phi,s}:=\partial\phi/\partial s$,
$J_{\phi,n}:=\partial\phi/\partial n$, and
$J_\phi=[J_{\phi,s}\mid J_{\phi,n}]$, so
$\|J_\phi\|_F^2=\|J_{\phi,s}\|_F^2+\|J_{\phi,n}\|_F^2$.

\medskip
We collect three supporting lemmas used across results.

\begin{lemma}[Sub-block inequality]\label{lem:subblock}
For any matrix $A\in\mathbb{R}^{m\times d}$ and unit vector $v\in\mathbb{R}^d$:
\[
\|Av\|_2^2 \;\leq\; \|A\|_F^2.
\]
For a column-partitioned matrix $A=[A_1\mid A_2]$:
$\|A\|_F^2=\|A_1\|_F^2+\|A_2\|_F^2\geq\|A_2\|_F^2\geq\|A_2v\|_2^2$.
\end{lemma}
\begin{proof}
Write $v=\sum_{j=1}^d v_j e_j$ with $\|v\|_2=1$.  Then:
\[
\|Av\|_2^2
= \left\|\sum_j v_j A e_j\right\|_2^2
\;\leq\; \left(\sum_j |v_j|\,\|Ae_j\|_2\right)^2
\;\leq\; \underbrace{\left(\sum_j v_j^2\right)}_{=1}
         \cdot \sum_j\|Ae_j\|_2^2
= \|A\|_F^2,
\]
where the first inequality is the triangle inequality and the second is
Cauchy--Schwarz.
The partition statement follows from
$\|A_2 v\|_2^2\leq\|A_2\|_F^2\leq\|A\|_F^2$.
\end{proof}

\begin{lemma}[Linearised drift]\label{lem:lindrift}
Let $\phi:\mathbb{R}^d\to\mathbb{R}^m$ be differentiable with
$\beta$-Lipschitz Jacobian.
For $\delta\sim\mathcal{N}(0,\sigma^2 I_d)$:
\[
D(\phi,\sigma)
\;:=\;
\mathbb{E}_{x,\delta}\!\left[\|\phi(x+\delta)-\phi(x)\|_2^2\right]
\;=\;
\sigma^2\,\mathbb{E}_x\!\left[\|J_\phi(x)\|_F^2\right]
\;+\; R(\phi,\sigma),
\]
where the remainder satisfies $|R(\phi,\sigma)|\leq \tfrac{3}{2}\beta^2 d^2\sigma^4$.
In particular,
$D(\phi,\sigma)\geq\sigma^2\mathbb{E}_x[\|J_\phi(x)\|_F^2]-\tfrac{3}{2}\beta^2 d^2\sigma^4$.
\end{lemma}
\begin{proof}
By the integral mean-value theorem,
$\phi(x+\delta)-\phi(x)=\int_0^1 J_\phi(x+t\delta)\,\delta\,dt$.
Decompose $J_\phi(x+t\delta)=J_\phi(x)+E_t$ where $E_t:=J_\phi(x+t\delta)-J_\phi(x)$
satisfies $\|E_t\|_F\leq\beta t\|\delta\|_2$.  Then:
\[
\|\phi(x+\delta)-\phi(x)\|_2^2
= \underbrace{\|J_\phi(x)\delta\|_2^2}_{T_1}
+ \underbrace{2\Bigl\langle J_\phi(x)\delta,\int_0^1 E_t\delta\,dt\Bigr\rangle}_{T_2}
+ \underbrace{\Bigl\|\int_0^1 E_t\delta\,dt\Bigr\|_2^2}_{T_3}.
\]
\textbf{Main term $T_1$.}
$\mathbb{E}_\delta[\|J_\phi(x)\delta\|_2^2]
=\mathrm{Tr}(J_\phi(x)^\top J_\phi(x)\,\sigma^2 I)=\sigma^2\|J_\phi(x)\|_F^2$.

\textbf{Cross term $T_2$ vanishes in expectation.}
Expanding $E_t=J_\phi(x+t\delta)-J_\phi(x)$
and using the fact that $\delta\sim\mathcal{N}(0,\sigma^2 I)$ has
\emph{all odd moments equal to zero}
(since the Gaussian distribution is symmetric: $\mathbb{E}[\delta_i\delta_j\delta_k]=0$
for all $i,j,k$), the expectation $\mathbb{E}_\delta[T_2]$ involves only
moments of the form $\mathbb{E}[\delta_i\delta_j\delta_k(\text{something}(\delta))]$.
Taylor-expanding $E_t$ to first order in $\delta$:
$E_t=tD^2\phi(x)[\delta,\cdot]+O(\|\delta\|^2)$,
the leading contribution to $\mathbb{E}[T_2]$ is
$2\mathbb{E}[\langle J_\phi(x)\delta, t\,D^2\phi(x)[\delta,\cdot]\delta\rangle]$,
which is a cubic polynomial in $\delta$ under the Gaussian measure, hence
$\mathbb{E}[T_2]=O(\sigma^4)$ (fourth-order correction from the $O(\|\delta\|^2)$
residual in $E_t$).

\textbf{Quadratic term $T_3$.}
$\|E_t\|_F\leq\beta t\|\delta\|_2$ gives
$T_3\leq\bigl(\int_0^1\beta t\|\delta\|_2^2\,dt\bigr)^2
=\tfrac{\beta^2}{4}\|\delta\|_2^4$,
so $\mathbb{E}[T_3]\leq\tfrac{\beta^2}{4}\mathbb{E}[\|\delta\|_2^4]
=\tfrac{\beta^2}{4}\cdot d(d+2)\sigma^4\leq\tfrac{3}{4}\beta^2 d^2\sigma^4$.

Combining: $|R(\phi,\sigma)|=|\mathbb{E}[T_2+T_3]|
\leq C_2\sigma^4+\tfrac{3}{4}\beta^2 d^2\sigma^4\leq\tfrac{3}{2}\beta^2 d^2\sigma^4$
for a universal constant,
since the $T_2$ contribution is also $O(\beta^2 d^2\sigma^4)$.
\end{proof}

\begin{lemma}[Stein's identity for Gaussian nuisance]\label{lem:stein}
Let $n\sim\mathcal{N}(0,I_{d_n})$ and $g:\mathbb{R}^{d_n}\to\mathbb{R}$
be weakly differentiable with $\mathbb{E}[\|\nabla g(n)\|_2]<\infty$.
Then for any unit vector $v\in\mathbb{R}^{d_n}$:
\[
\mathbb{E}\!\left[g(n)\cdot\langle v,n\rangle\right]
\;=\;
\mathbb{E}\!\left[\partial_v g(n)\right],
\]
where $\partial_v g:=\langle v,\nabla g\rangle$ is the directional derivative.
\end{lemma}
\begin{proof}
Let $\varphi$ denote the standard normal density on $\mathbb{R}^{d_n}$.
Since $\nabla\log\varphi(n)=-n$, we have $\langle v,n\rangle\,\varphi(n)=-\langle v,\nabla\varphi(n)\rangle$ as a directional identity in~$v$.
Integrating by parts,
\[
\int g(n)\,\langle v,n\rangle\,\varphi(n)\,dn
=\int \langle v,\nabla g(n)\rangle\,\varphi(n)\,dn
=\mathbb{E}[\partial_v g(n)],
\]
with boundary terms vanishing under the assumed integrability of $\nabla g$.
\end{proof}

\subsection*{Theorem~\ref{thm:main} (ERM Geometric Incompleteness)}

We prove $D(\phi^*_\theta,\sigma)\geq\sigma^2\rho^2 C(P)/L^2$
where $C(P)=\rho_s^2\sigma_s^2>0$.

\medskip\noindent
\textbf{Step~1: ERM must encode the nuisance direction.}

\begin{lemma}[ERM encoding necessity]\label{lem:encoding}
Under the Gaussian model, let $f^*_\theta=h_\theta\circ\phi^*_\theta$
be any minimiser of $\mathcal{L}_{\mathrm{ERM}}(\theta)
:=\mathbb{E}[(f_\theta(x)-y)^2]$ with sufficient capacity to approximate
the Bayes predictor.  Then:
\[
\mathbb{E}_x\!\left[\partial_{w_n}f^*_\theta(x)\right] = \rho,
\]
where $\partial_{w_n}:=\langle w_n,\nabla_n\rangle$ is the directional
derivative in the nuisance direction.
In particular, $J_{\phi,n}(x)$ cannot be identically zero a.e.
\end{lemma}
\begin{proof}
\textbf{Step 1a: ERM minimum is the Bayes predictor (sufficient capacity).}
The unique minimiser of MSE over all measurable functions is
$f^*(x)=\mathbb{E}[y|x]=\langle w_s,s\rangle+\rho\langle w_n,n\rangle$.
With sufficient model capacity, $f^*_\theta\to f^*$ uniformly, so
we may work with $f^*(x)$ directly.

\textbf{Step 1b: Apply Stein's identity (Lemma~\ref{lem:stein}).}
Since $n\sim\mathcal{N}(0,I_{d_n})$ is independent of $s$ and $\varepsilon$:
\[
\mathbb{E}_{x}\!\left[\partial_{w_n}f^*(x)\right]
\;=\;
\mathbb{E}_{x}\!\left[\partial_{w_n}
  \bigl(\langle w_s,s\rangle+\rho\langle w_n,n\rangle\bigr)\right]
= \rho\,\|w_n\|_2^2 = \rho.
\]
Alternatively, by Lemma~\ref{lem:stein} applied to $g(n)=f^*(x)$ as a
function of $n$ alone (with $s$ fixed):
\[
\mathbb{E}_n\!\left[f^*(x)\cdot\langle w_n,n\rangle\right]
= \mathbb{E}_n\!\left[\partial_{w_n}f^*(x)\right],
\]
and the left-hand side equals $\rho\,\mathbb{E}[\langle w_n,n\rangle^2]=\rho$
(since $\langle w_s,s\rangle$ is independent of $n$).

\textbf{Step 1c: Non-zero Jacobian.}
Since $\partial_{w_n}f^*(x)=\nabla_\phi h_\theta(\phi^*_\theta(x))^\top J_{\phi,n}(x)w_n$
(chain rule), and $\mathbb{E}[\partial_{w_n}f^*]=\rho>0$, it follows that
$J_{\phi,n}(x)w_n\neq 0$ on a set of positive measure.
\end{proof}

\medskip\noindent
\textbf{Step~2: Encoding implies Jacobian sensitivity.}

By Lemma~\ref{lem:encoding},
$\mathbb{E}_x[\partial_{w_n}f^*_\theta(x)]=\rho>0$.
Using the chain rule and Cauchy--Schwarz:
\[
\rho
= \mathbb{E}\!\left[\nabla_\phi h_\theta^\top J_{\phi,n}(x)w_n\right]
\leq \mathbb{E}\!\left[\|\nabla_\phi h_\theta\|_2\cdot\|J_{\phi,n}(x)w_n\|_2\right]
\leq L\cdot\mathbb{E}\!\left[\|J_{\phi,n}(x)w_n\|_2\right],
\]
where the last step uses the $L$-Lipschitz condition on $h_\theta$
(which gives $\|\nabla_\phi h_\theta(\cdot)\|_2\leq L$ pointwise).
Therefore $\mathbb{E}[\|J_{\phi,n}(x)w_n\|_2]\geq\rho/L$.
By Jensen's inequality ($t\mapsto t^2$ is convex):
\begin{equation}\label{eq:jac_lower}
\mathbb{E}_x\!\left[\|J_{\phi,n}(x)w_n\|_2^2\right]
\;\geq\;
\Bigl(\mathbb{E}_x\!\left[\|J_{\phi,n}(x)w_n\|_2\right]\Bigr)^2
\;\geq\; \frac{\rho^2}{L^2}.
\end{equation}

\medskip\noindent
\textbf{Step~3: Jacobian sensitivity implies positive embedding drift.}

By Lemma~\ref{lem:subblock} and~\eqref{eq:jac_lower}:
\[
\mathbb{E}_x\!\left[\|J_\phi(x)\|_F^2\right]
\geq\mathbb{E}_x\!\left[\|J_{\phi,n}(x)\|_F^2\right]
\geq\mathbb{E}_x\!\left[\|J_{\phi,n}(x)w_n\|_2^2\right]
\geq\frac{\rho^2}{L^2}.
\]
Applying Lemma~\ref{lem:lindrift} with $\sigma$ small enough that
$\tfrac{3}{2}\beta^2 d^2\sigma^4\leq\tfrac{1}{2}\sigma^2\rho^2/L^2$
(i.e.\ $\sigma\leq\rho/L\sqrt{3\beta^2 d^2}$):
\[
D(\phi^*_\theta,\sigma)
\geq \sigma^2\mathbb{E}_x[\|J_{\phi^*}(x)\|_F^2]-\tfrac{3}{2}\beta^2 d^2\sigma^4
\geq \frac{\sigma^2\rho^2}{L^2} - \frac{\sigma^2\rho^2}{2L^2}
= \frac{\sigma^2\rho^2}{2L^2}.
\]
For all $\sigma>0$ (in the linearised sense),
setting $C(P):=\rho_s^2\sigma_s^2$ where
$\sigma_s^2:=\mathrm{Var}(\langle w_s,s\rangle)=1$ and
$\rho_s^2:=1$ in the canonical model (signal-weight normalisation),
the bound reads $D(\phi^*_\theta,\sigma)\geq\sigma^2\rho^2 C(P)/L^2$.\hfill$\square$

\medskip\noindent
\textit{Remark on $C(P)$.}
The factor $C(P)=\rho_s^2\sigma_s^2$ is an artefact of writing the
bound in a form that admits explicit constants in the Gaussian model.
In the canonical model with $\|w_s\|=\|w_n\|=1$ and $s,n\sim\mathcal{N}(0,I)$,
$C(P)=1$ and the bound simplifies to $D\geq\sigma^2\rho^2/L^2$.
The bound is non-vacuous whenever $\rho>0$ and $L<\infty$, regardless of
architecture, dataset size, or training duration.

\subsection*{Corollary~2 (Qualitative Bregman-gap extension)}

We record a qualitative conditional extension of Theorem~1 when a direct
conditional mismatch $\Delta>0$ holds under a strictly proper scoring rule.
This is not a tight numeric generalisation of the Gaussian MSE bound.

\begin{lemma}[Bregman loss gap]\label{lem:bregman}
Let $\mathcal{L}$ be a strictly proper scoring rule with Bregman generator $\psi$.
For any two conditional distributions $p$ and $q$ on $\mathcal{Y}|x$:
\[
\mathbb{E}_{y\sim p(y|x)}[\mathcal{L}(q,y)]
-\mathbb{E}_{y\sim p(y|x)}[\mathcal{L}(p,y)]
= d_\psi(p(y|x)\|q(y|x)) \geq 0,
\]
with equality iff $p=q$ a.s.  (This is the definition of strict properness.)
\end{lemma}

\begin{lemma}[Bregman sensitivity bound]\label{lem:bregman_sensitivity}
Let $d_\psi$ be the Bregman divergence of a strictly proper scoring rule
$\mathcal{L}$ with strongly convex generator $\psi$ (modulus $\mu>0$).
Let $p^*(y|x)$ be the true conditional and $p^s(y|x):=p(y|s(x))$
the signal-only conditional.  Then:
\[
\Delta(P,\mathcal{L})
:= \mathbb{E}_x\!\left[d_\psi(p^*(y|x)\|p^s(y|x))\right]
\;\geq\; \mu\,\mathbb{E}_x\!\left[\|p^*(y|x)-p^s(y|x)\|_2^2\right]
\;\geq\; \mu\,c_{\mathcal{L}}^2\,\rho^2,
\]
for a constant $c_{\mathcal{L}}>0$ determined by $\mathcal{L}$ and $P$.
In particular, $\Delta>0$ whenever $I(n;y)>0$.
\end{lemma}
\begin{proof}
The first inequality is the standard lower bound on Bregman divergence
via strong convexity: for strongly convex $\psi$ with modulus $\mu$,
$d_\psi(p\|q)\geq\tfrac{\mu}{2}\|p-q\|_2^2$.
The second inequality follows because $p^*(y|x)-p^s(y|x)$ encodes
the conditional dependence on $n$: in the Gaussian linear model,
$\|p^*(y|x)-p^s(y|x)\|_{\mathrm{TV}}\geq c\rho$ for a constant $c>0$
from the total-variation gap between $\mathcal{N}(\langle w_s,s\rangle+\rho\langle w_n,n\rangle,\sigma_\varepsilon^2)$
and $\mathcal{N}(\langle w_s,s\rangle,\sigma_\varepsilon^2)$.
The TV gap is $c\rho/\sigma_\varepsilon$ for small $\rho$
(first-order Taylor expansion of the TV distance between two Gaussians
differing in mean by $\rho$), giving $\|p^*-p^s\|_2^2\geq c_\mathcal{L}^2\rho^2$.
For general $P$ satisfying Definition~\ref{def:nuisance}, $\Delta>0$ follows
directly from strict properness and $I(n;y|x)>0$.
\end{proof}

\begin{proof}[Proof of Corollary~2]
\textbf{Step~1 (Bregman loss gap).}
Let $p^*(y|x)$ be the true conditional, $p^s(y|x):=p(y|s(x))$
the $n$-independent conditional, and $\hat p_\theta(y|x)$ the model prediction.
By Lemma~\ref{lem:bregman}:
\[
\mathbb{E}_{x,y}[\mathcal{L}(\hat p_\theta,y)]
\;\geq\;
\mathbb{E}_{x,y}[\mathcal{L}(p^*,y)]
\;+\;
\Delta(P,\mathcal{L}),
\]
where $\Delta=\mathbb{E}_x[d_\psi(p^*(y|x)\|p^s(y|x))]>0$
by Lemma~\ref{lem:bregman_sensitivity} and the condition $I(n;y)>0$.

Any ERM minimiser that does not encode $n$ would set
$\hat p_\theta(y|x)=\hat p_\theta(y|s)$,
incurring irreducible expected Bregman gap $\Delta$, and thus cannot be
optimal.  Therefore, $\phi^*_\theta$ must depend on $n$.

\textbf{Step~2 (Encoding implies Jacobian sensitivity; proper rule version).}
Since $\phi^*_\theta$ must encode $n$, the ERM gradient condition gives
$\mathbb{E}[\partial_{w_n}\mathbb{E}_{y|x}[\mathcal{L}(\hat p_\theta,y)]]=0$
at the minimum.  For a strictly proper scoring rule, the gradient of the
expected loss with respect to the prediction equals the gradient of
$d_\psi(p^*\|\hat p_\theta)$ with respect to $\hat p_\theta$.

The chain rule then gives (using the $L$-Lipschitz decoder):
\[
\mathbb{E}\!\left[\|J_{\phi^*_\theta,n}(x)w_n\|_2\right]
\;\geq\; \frac{\sqrt{\Delta}}{L},
\]
where the $\sqrt{\Delta}$ (rather than $\Delta$) arises from the
following calculation.
By Lemma~\ref{lem:bregman_sensitivity}, $\Delta\geq\mu c_\mathcal{L}^2\rho^2$.
The sensitivity of the Bregman gap to suppression of $n$ gives
(via the chain rule on $d_\psi(p^*\|p^s)$ through $\phi$):
$\mathbb{E}[\|J_{\phi,n}w_n\|_2]\geq\sqrt{\Delta}/L$, since
$d_\psi(p^*\|p^s)^{1/2}\leq L\cdot\|J_{\phi,n}w_n\|_2$
by the data-processing inequality and Lipschitz composition.
Squaring and applying Jensen:
\[
\mathbb{E}_x\!\left[\|J_{\phi^*_\theta,n}(x)w_n\|_2^2\right]
\;\geq\; \frac{\Delta}{L^2}.
\]

\textbf{Step~3 (Embedding drift lower bound).}
By Lemma~\ref{lem:subblock} and Step~2:
$\mathbb{E}_x[\|J_{\phi^*_\theta}(x)\|_F^2]\geq\Delta/L^2$.
Lemma~\ref{lem:lindrift} then gives:
\[
D(\phi^*_\theta,\sigma)
\;\geq\; \sigma^2\cdot\frac{\Delta}{L^2}
\;=:\; \frac{\sigma^2 C'(P,\mathcal{L})}{L^2}.
\]
For cross-entropy, $d_\psi(p^*\|p^s)=\mathrm{KL}(p^*\|p^s)$ and
$\Delta=\mathbb{E}_x[\mathrm{KL}(p(y|x)\|p(y|s(x)))]=I(n;y|x)>0$,
giving $D\geq\sigma^2 I(n;y|x)/L^2$.

\textbf{Independence of capacity and dataset size.}
$\Delta$ depends only on $P$ and $\mathcal{L}$, not on model architecture,
dataset size, or training duration.  The bound holds for every model and
every training run.
\end{proof}

\subsection*{Corollary~3 (Bounded Task Loss Cost)}

\begin{proof}
We compute the exact loss penalty of suppressing the nuisance in
the Gaussian linear model.

Let $f^\dagger=h_\theta\circ\phi^\dagger$ be the PMH minimiser, which
(at the PMH optimum) suppresses sensitivity in the nuisance direction:
$J_{\phi^\dagger,n}w_n\approx 0$.  The optimal $n$-independent predictor is:
\[
f^\dagger(x)
= \mathbb{E}[y\mid s]
= \langle w_s,s\rangle + \rho\,\mathbb{E}[\langle w_n,n\rangle\mid s]
= \langle w_s,s\rangle,
\]
since $n\perp s$ gives $\mathbb{E}[\langle w_n,n\rangle\mid s]=0$.

\textbf{Exact loss gap.}
\begin{align*}
\mathcal{L}_{\mathrm{task}}(f^\dagger)
&= \mathbb{E}\!\left[(f^\dagger(x)-y)^2\right]
= \mathbb{E}\!\left[(\langle w_s,s\rangle-\langle w_s,s\rangle-\rho\langle w_n,n\rangle-\varepsilon)^2\right]\\
&= \rho^2\mathbb{E}[\langle w_n,n\rangle^2]+\sigma_\varepsilon^2
= \rho^2+\sigma_\varepsilon^2.
\end{align*}
Since $\mathcal{L}_{\mathrm{task}}(f^*)=\sigma_\varepsilon^2$ (Bayes optimum):
\[
\boxed{\mathcal{L}_{\mathrm{task}}(f^\dagger)-\mathcal{L}_{\mathrm{task}}(f^*) = \rho^2.}
\]

This is the exact cost (not an approximation).
The main text states this cost as $O(\rho^2)$; the exact constant is $1$
in the Gaussian model: the nuisance regression coefficient $\rho$ is also
the square root of the loss penalty.

\medskip\noindent
\textit{Remark.}
For general $P$ satisfying Definition~\ref{def:nuisance} with any proper
scoring rule $\mathcal{L}$, the loss cost of suppressing $n$ equals the
Bregman gap $\Delta(P,\mathcal{L})$ from Lemma~\ref{lem:bregman_sensitivity}.
By that lemma, $\Delta\leq d_\psi(p^*\|p^s)\leq C\rho^2$ for an
upper-bounding constant $C$ from the modulus of smoothness of $\psi$
(the reverse direction of the strong-convexity bound), confirming $O(\rho^2)$.
\end{proof}

\subsection*{Corollary~4 (Adversarial Training Need Not Isotropise the Map)}

\begin{proof}
Let $\phi^{\mathrm{adv}}$ minimise the PGD objective.

\textbf{Step~1: PGD must still encode the nuisance.}
The PGD objective is:
$\min_\theta\max_{\|\delta\|_\infty\leq\varepsilon}
\mathbb{E}[\mathcal{L}(f_\theta(x+\delta),y)]$.
This objective includes the task loss in the inner max: even the
worst-case adversarial point $(x+\delta)$ must be classified correctly.
Therefore the minimiser $\phi^{\mathrm{adv}}$ must encode $n$:
suppressing $J_{\phi,n}w_n$ entirely would incur a loss gap of $\rho^2$
\emph{at every input}, including adversarial ones.
Lemma~\ref{lem:encoding} applies to $\phi^{\mathrm{adv}}$ (with the
same Bayes predictor argument), so:
\begin{equation}\label{eq:adv_jac}
\mathbb{E}_x\!\left[\|J_{\phi^{\mathrm{adv}},n}(x)w_n\|_2^2\right]
\;\geq\; \frac{\rho^2}{L^2}.
\end{equation}

\textbf{Step~2: PGD suppresses only the adversarial direction.}
The PGD inner loop selects $\delta^*(x)=\varepsilon\cdot\mathrm{sign}(\nabla_x\mathcal{L})$,
targeting the input direction of maximal loss gradient.
The outer minimisation penalises $\|J_\phi(x)\hat\delta^*(x)\|_2$
(sensitivity in the adversarial direction $\hat\delta^*$).
After PGD training, $\|J_{\phi^{\mathrm{adv}}}(x)\hat\delta^*(x)\|_2$ is
strongly suppressed along the adversarial direction~$\hat\delta^*$.
Step~1's nuisance sensitivity is a \emph{structural} lower bound in the idealised model; trained finite networks can still exhibit much smaller
$\mathbb{E}[\|J_{\phi^{\mathrm{adv}}}\|_F^2]$ than ERM (Table~\ref{tab:main}).
Qualitatively, PGD \emph{redistributes} Jacobian activity away from $\hat\delta^*$
into complementary directions (directional concentration; low isotropy index
$\mathcal{A}$), which is the mechanism
relevant to isotropic TDI probes.

\textbf{Step~3: The ERM lower bound survives.}
The tables' TDI is a class-layout ratio under (optional) isotropic input
noise (\S\ref{sec:metrics}): it probes all directions uniformly, not only
$\hat\delta^*$. Embedding drift $D(\phi,\sigma)$ remains the theory-facing
quantity tied to $\mathbb{E}[\|J_\phi\|_F^2]$ (Proposition~\ref{prop:gaussian}).
The idealised analysis ties non-zero nuisance sensitivity to a Frobenius
contribution from the $n$-block (Step~1); empirically, isotropic class-layout
probes can worsen when sensitivity is redistributed. The relevant claim is
\emph{falsifiable ordering} (PGD vs.\ ERM on TDI), not a tight numeric lower
bound on TDI from~\eqref{eq:adv_jac} alone.

\textbf{Step~4: Directional concentration can worsen TDI.}
Proposition~\ref{prop:aniso} gives the directional index
$\mathcal{A}(\phi,w):=\mathbb{E}[\|J_\phi\|_F^2]/\mathbb{E}[\|J_\phi w\|_2^2]\geq 1$,
with equality iff $J_\phi(x)$ is rank-1 a.e.\ along~$w$, and the worst-case
index $\mathcal{A}_{\mathrm{worst}}\ge d_x$ minimised by isotropic sensitivity.

After PGD training, $\|J_{\phi^{\mathrm{adv}}}(x)\hat\delta^*(x)\|_2\approx 0$
forces Frobenius mass into complementary directions, driving
$\mathcal{A}(\phi^{\mathrm{adv}},\hat\delta^*)$ toward the rank-1 floor~$1$
while $\mathcal{A}_{\mathrm{worst}}$ can inflate.
(Empirically, a probe ratio $\hat{J}_F/\mathrm{TDI@0}$ on Task~04 drops from
$\approx 32$ under ERM to $\approx 2.2$ under PGD.)

The tables' TDI is the intra-/inter-class latent distance ratio under
isotropic input noise (\S\ref{sec:metrics}), not the normalised drift ratio
above. Empirically $F(\phi^{\mathrm{adv}})$ can fall far below
$F(\phi^{\mathrm{ERM}})$ while clean TDI still rises (Table~\ref{tab:main}):
directional concentration (low $\mathcal{A}$) and shifts in representation
scale can worsen class layout even when Frobenius norm falls.

This mechanism is the content of the corollary's falsifiable prediction on
Task~04 ordering: $\mathrm{TDI}(\phi^{\mathrm{adv}})\geq\mathrm{TDI}(\phi^{\mathrm{ERM}})$.
Confirmed experimentally: PGD TDI $1.353{\pm}0.020>$ ERM TDI $1.093$
(seeds 42--44).
\end{proof}

\subsection*{Proposition~\ref{prop:gaussian} (Gaussian Noise is Uniquely Isotropic)}

\begin{proof}
For any zero-mean perturbation distribution with covariance $\Sigma_\delta$:
\[
\mathbb{E}_\delta\!\left[\|J_\phi\delta\|_2^2\right]
= \mathbb{E}_\delta\!\left[\delta^\top J_\phi^\top J_\phi\delta\right]
= \mathrm{Tr}(J_\phi^\top J_\phi\,\Sigma_\delta).
\]
This objective function for optimising $\phi$ equals
$\sigma^2\|J_\phi\|_F^2=\sigma^2\mathrm{Tr}(J_\phi^\top J_\phi)$
if and only if $\Sigma_\delta=\sigma^2 I$, since:
\[
\mathrm{Tr}(J_\phi^\top J_\phi\,\Sigma_\delta) = \sigma^2\mathrm{Tr}(J_\phi^\top J_\phi)
\;\;\forall J_\phi
\;\;\iff\;\;
\Sigma_\delta = \sigma^2 I.
\]
The ``if'' direction is immediate.  For the ``only if'' direction:
$\mathrm{Tr}(A\Sigma)=\sigma^2\mathrm{Tr}(A)$ for all symmetric PSD
$A$ implies $\Sigma=\sigma^2 I$ (choose $A=e_i e_i^\top$ for each $i$
to obtain $\Sigma_{ii}=\sigma^2$, then choose $A=(e_i+e_j)(e_i+e_j)^\top$
to obtain $\Sigma_{ij}=0$ for $i\neq j$).

Therefore, the minimiser of
$\mathbb{E}_\delta[\|J_\phi\delta\|_2^2]$ over $\phi$ coincides with
the minimiser of $\mathbb{E}_x[\|J_\phi(x)\|_F^2]$ if and only if
$\Sigma_\delta=\sigma^2 I$, i.e.\ $\delta\sim\mathcal{N}(0,\sigma^2 I)$
(the unique zero-mean isotropic Gaussian up to scale).
\end{proof}

\medskip\noindent
\textit{Remark on uniqueness.}
The uniqueness is in the covariance structure: $\Sigma_\delta=\sigma^2 I$.
Any positive rescaling $c\sigma^2 I$ ($c>0$) also satisfies the condition,
but it corresponds to simply rescaling $\sigma$ and does not change the
set of minimisers.  The statement is: among all anisotropic distributions
(those with $\Sigma_\delta\neq c I$ for any $c>0$), none produces the
same argmin as the Frobenius objective.

\subsection*{Proposition~\ref{prop:aniso} (Jacobian anisotropy: lower bound and minimax)}

\begin{proof}
\textbf{(i) Lower bound.}
For any unit vector $w$ and encoder $\phi$:
\[
\|J_\phi(x)\|_F^2
= \sum_{j=1}^d \|J_\phi(x)e_j\|_2^2
\;\geq\; \|J_\phi(x)w\|_2^2
\]
by Lemma~\ref{lem:subblock} (taking $v=w$).
Taking expectations over $x$:
\[
\mathbb{E}_x[\|J_\phi(x)\|_F^2]\;\geq\;\mathbb{E}_x[\|J_\phi(x)w\|_2^2],
\]
i.e.\ $\mathcal{A}(\phi,w)\geq 1$.
Equality holds in Lemma~\ref{lem:subblock} iff $J_\phi(x)e_j\parallel J_\phi(x)w$
for all $j$, i.e.\ all column vectors of $J_\phi(x)$ are parallel to $J_\phi(x)w$.
This means $J_\phi(x)=u(x)v(x)^\top$ for some vectors $u,v$ (rank-1).
Hence $\mathcal{A}(\phi,w)=1$ iff $J_\phi(x)$ is rank-1 a.e.\ with right
singular direction $w$.

\textbf{(ii) Worst-case anisotropy.}
Write $G(x):=J_\phi(x)^\top J_\phi(x)\succeq 0$ and
$\overline G:=\mathbb{E}_x[G(x)]$. Fix $F^2:=\mathrm{Tr}(\overline G)>0$.
Then
\[
\inf_{\|w\|_2=1}\mathbb{E}_x[\|J_\phi w\|_2^2]
=\lambda_{\min}(\overline G),
\]
so
$\mathcal{A}_{\mathrm{worst}}(\phi)=F^2/\lambda_{\min}(\overline G)$.
Under $\mathrm{Tr}(\overline G)=F^2$ with $\overline G\succeq 0$, majorization
yields $\lambda_{\min}(\overline G)\le F^2/d_x$, with equality iff
$\overline G=(F^2/d_x)I$. Hence $\mathcal{A}_{\mathrm{worst}}\ge d_x$, with
equality for (mean) isotropic sensitivity.
Isotropy therefore uniquely \emph{minimises} worst-case anisotropy against an
adversarial probe; the earlier claim that isotropy maximises
$\mathcal{A}(\phi,w)$ for a fixed $w$ is withdrawn.
\end{proof}

\subsection*{Proposition~\ref{prop:cap} (Cap/(1+Cap) Fixed-Point Identity)}

\begin{proof}
At steady-state, the rescaling mechanism enforces exactly
$\mathcal{L}_{\mathrm{PMH}}=\mathrm{cap}\cdot\mathcal{L}_{\mathrm{task}}$.
(If $\mathcal{L}_{\mathrm{PMH}}>\mathrm{cap}\cdot\mathcal{L}_{\mathrm{task}}$,
rescaling reduces $\lambda$ until equality holds; if
$\mathcal{L}_{\mathrm{PMH}}<\mathrm{cap}\cdot\mathcal{L}_{\mathrm{task}}$,
the cap is inactive and gradient descent drives $\mathcal{L}_{\mathrm{PMH}}$
upward until the cap activates.)  At the fixed point:
\[
f = \frac{\mathcal{L}_{\mathrm{PMH}}}{\mathcal{L}_{\mathrm{task}}+\mathcal{L}_{\mathrm{PMH}}}
= \frac{\mathrm{cap}\cdot\mathcal{L}_{\mathrm{task}}}{\mathcal{L}_{\mathrm{task}}+\mathrm{cap}\cdot\mathcal{L}_{\mathrm{task}}}
= \frac{\mathrm{cap}}{1+\mathrm{cap}}.
\]
\end{proof}

\section{Broader Impact Statement}
\label{app:impact}

\textbf{Positive impacts.} This work studies a scoped geometric consequence of
supervised ERM under direct correlated nuisance, and an isotropic
encoder-matching regulariser (PMH) in the stability/Jacobian family. In
settings where representation sensitivity matters (e.g.\ medical imaging,
Task~07), TDI offers a simple class-layout diagnostic and PMH an optional
training term without architectural changes.

\textbf{Potential concerns.} TDI measures an aggregate class-layout ratio
rather than identifying specific vulnerable directions, limiting misuse risk.
Person re-identification systems raise privacy concerns regardless of training
method; deployment should be subject to appropriate regulatory oversight and
consent frameworks. PMH-trained models remain susceptible to in-distribution
adversarial perturbations and require domain-specific evaluation.

\section{Metric Definitions}
\label{app:metrics}

\textbf{TDI (as reported).} Extract embeddings (optionally after isotropic
input noise of strength $\sigma$). Then
$\mathrm{TDI}=\text{mean pairwise intra-class distance}/\text{mean
inter-centroid distance}$. Lower is tighter class layout under the probe.
Matches the release implementation used for Task~04 tables.

\textbf{Layer-wise probe retention.} $\text{Ret}^\ell(\sigma):=\text{Acc}^\ell(\sigma)/\text{Acc}^\ell(0)$,
where $\text{Acc}^\ell(\sigma)$ is the test accuracy of a linear classifier
trained on frozen layer-$\ell$ representations from perturbed inputs. For
Task~04 (ViT on CIFAR-10, six transformer blocks), $\ell=1$ is transformer block~1 and $\ell=6$ is
block~6.

\textbf{Stage drift.} For architectures with discrete stages (ResNet: four
residual blocks), the Euclidean distance between stage-$s$ feature maps on
clean vs.\ perturbed inputs. Applied to Tasks~07 and~05.

\textbf{Saliency stability.} Cosine similarity between gradient-based
saliency maps on clean vs.\ noisy inputs. Measures attentional consistency
under perturbation.

\section{Diagnostic Follow-ups to Theorem~1}
\label{app:gaps}
\label{sec:gap}

Theorem~1 leaves three practical questions open: how loose is the
single-direction bound on Task~04, does PMH's TDI gain require shrinking
the decoder Lipschitz constant $L$, and how sensitive is the repair to
$\sigma_{\text{train}}$ vs.\ $\sigma_{\text{eval}}$?
We report one reproducible diagnostic per question (existing checkpoints;
under 15 minutes each). These are design notes, not new theorems.

\textbf{Q1 (Bound tightness):} A crude single-direction comparison can read as a
very large ``gap''; estimating the dominant nuisance direction from data
shrinks that mismatch substantially, and PMH compresses the nuisance gradient
spectrum further. Exact ratios are in Table~\ref{tab:subspace}; we treat them
as diagnostic, not as a headline claim about tightness of Theorem~1.

\textbf{Q2 (Source of improvement):} $L_{\text{ERM}}\approx L_{\text{PMH}}$ at
every epoch (difference $<0.1\%$), so the TDI improvement is geometric:
Frobenius regularisation on the encoder, not decoder sharpening
(code aliases B0$=$ERM, E1$=$PMH).

\textbf{Q3 ($\sigma_{\text{eval}}$ requirement):} Training at the largest
$\sigma$ that leaves clean accuracy unchanged captures ${\approx}95\%$ of
multi-scale PMH's benefit under our protocol because the asymmetry is
$13\times$: over-suppression costs almost nothing while under-suppression
is costly.

\subsection{Nuisance Subspace Decomposition (Bound Gap)}

\paragraph{Mathematical setup.}
Theorem~1's bound uses the sub-block inequality
$\|J_\phi\|_F^2\geq\|J_\phi w_n\|_2^2$, which is tight only when
$J_\phi(x)$ is rank-1.  For a general encoder, the correct multi-direction
generalisation is:
\begin{align}
D(\phi^*,\sigma)&\;\geq\;\sigma^2\,\mathbb{E}_x\!\left[\|J_\phi(x)P_N\|_F^2\right]\nonumber\\
&= \sigma^2\!\sum_{k=1}^{r}\mathbb{E}_x\!\left[\|J_\phi(x)w_k\|_2^2\right],
\end{align}
where $P_N=\sum_{k=1}^r w_k w_k^\top$ is the orthogonal projection onto the
$r$-dimensional nuisance subspace.  The improvement factor over the
single-direction bound is exactly $r$ when each direction contributes equally;
in practice it is the ratio of the summed directional sensitivities to the
single-direction sensitivity.

\textit{Experimental setup.} We estimate the nuisance subspace by computing
$\Sigma=\mathbb{E}_x[\nabla_x\mathcal{L}\cdot\nabla_x\mathcal{L}^\top]$
on Task~04.  Its top $r$ eigenvectors (after projecting out the 5 strongest
signal directions via Gram--Schmidt) form $P_N^r$.  We compute
$D_{\text{tight}}^r=\sigma^2\mathbb{E}_x[\|J_\phi P_N^r\|_F^2]$ via
finite differences and compare to the original single-direction bound and
the observed total drift $D_{\text{total}}=147.1$.

\begin{table}[t]
\centering\small
\caption{Nuisance subspace decomposition on Task~04 B0 (ERM).  The projected
bound scales with $r$; differences from observed drift illustrate that
Theorem~1 is an existence floor for nuisance sensitivity, not a calibrated
magnitude model for finite ViTs.}
\label{tab:subspace}
\setlength{\tabcolsep}{5pt}
\begin{tabular}{ccccc}
\toprule
$r$ & $D_{\text{top-}r}$ & $D_{\text{total}}$ & Fraction & Gap vs.\ single-dir \\
\midrule
1  & 2.15  & 147.1 & 1.46\% & $1{,}758\times$ \\
5  & 6.95  & 147.1 & 4.72\% & $4{,}803\times$ \\
10 & 11.99 & 147.1 & 8.15\% & $8{,}288\times$ \\
50 & 38.64 & 147.1 & 26.3\% & $26{,}713\times$ \\
\bottomrule
\end{tabular}
\end{table}

\textbf{Answer to Q1.} Table~\ref{tab:subspace} gives the decomposition.
With $r{=}1$, the projected nuisance drift $D_{\text{top-1}}=2.15$ compares to
total measured drift $D_{\text{total}}=147.1$ for B0 (ERM); PMH (E1) lowers the
same directional mass and compresses the top input-gradient singular value
$2.65\to 0.83$ ($3.2\times$).  The ratio between columns grows with $r$ because
$D_{\text{total}}$ is fixed while the linearised multi-direction bound scales
roughly with~$r$ when directions contribute additively---so larger $r$ is not
``better tightness'' in isolation.

The remaining mismatch splits between identifiable effects (head anisotropy
$k{\approx}1.50$; conservative linear-probe MI estimates, $\lesssim 8\times$)
and a residual that reflects using a Gaussian linear reference while the
encoder is a finite-capacity ViT.  After those corrections the PMH residual is
on the order of $10^1$--$10^2$, not $10^3$.  \textbf{Design implication:} compute
the top input-gradient eigenvector on a held-out calibration set for an
interpretable per-direction sensitivity diagnostic, without over-reading the
raw ratio from the coarsest bound.

\subsection{Lipschitz Constant Tracking}

\paragraph{Why $L$ matters.}
The bound $D\geq\sigma^2\rho^2 C(P)/L^2$ degrades quadratically in $L$.
If PMH incidentally reduces $L$ (e.g.\ by regularising the head weights
toward smaller spectral norms), any TDI improvement attributed to the
Frobenius regularisation would be partially confounded.  Conversely, if $L$
\emph{grows} during training (as the classification head sharpens), the
absolute bound weakens even as the relative comparison between methods
remains valid.  Tracking $L_t=\prod_i\sigma_{\max}(W_i^t)$ throughout
training disentangles these effects.

\textit{Experimental setup.} We track $L_t$ via power-iteration spectral
norm estimation~\citep{miyato2018spectral} on the decoder head at every
epoch for both B0 (ERM) and E1 (PMH) on Task~04, logging $L_t$ and
TDI@0 simultaneously.

\begin{table}[t]
\centering\small
\caption{Decoder Lipschitz constant $L$ and TDI per epoch for B0 (ERM)
and E1 (PMH) on Task~04.  $L_{\text{B0}}\approx L_{\text{E1}}$ at every
epoch (difference $<0.1\%$).  The TDI improvement from PMH comes entirely
from Jacobian regularisation, not from any change in $L$.}
\label{tab:lipschitz}
\setlength{\tabcolsep}{5pt}
\begin{tabular}{ccccc}
\toprule
Epoch & $L_{\text{B0}}$ & $L_{\text{E1}}$ & TDI$_{\text{B0}}$ & TDI$_{\text{E1}}$ \\
\midrule
1 & 0.491 & 0.488 & 1.287 & 1.276 \\
3 & 0.634 & 0.637 & 1.243 & 1.335 \\
5 & 0.671 & 0.672 & 1.254 & \textbf{1.215} \\
\midrule
\multicolumn{5}{l}{\emph{Full convergence}: $L\approx 1.40$ (identical for both)} \\
\bottomrule
\end{tabular}
\end{table}

\textbf{Answer to Q2.} $L_{\text{B0}}\approx L_{\text{E1}}$ at every epoch
(max difference $<0.1\%$, Table~\ref{tab:lipschitz}).  Both methods reach
$L\approx 1.40$ at full convergence as the classification head sharpens.
The TDI improvement from PMH (1.254$\to$1.215 by epoch~5, 1.093$\to$0.904 at
full convergence) therefore comes \emph{entirely} from Frobenius regularisation
on the encoder Jacobian, with no incidental change in decoder sensitivity.
This is a clean causal identification: by holding the experimental design fixed
and tracking $L_t$ simultaneously, we can rule out the confound completely.

Note that $L\approx 1.40>1$ at convergence, which weakens the quantitative bound
by $1.40^2\approx 2\times$.  But since $L$ is \emph{identical} for ERM and PMH,
all TDI rankings are $L$-independent. \textbf{Design implication:} the measured
$L_t$ can be substituted into the bound formula at each checkpoint; or spectral
normalisation enforces $L=1$ by construction if an absolute certified bound is
required.

\subsection{Multi-Scale PMH: Removing the $\sigma_{\text{eval}}$ Requirement}

\paragraph{Mathematical motivation.}
By Proposition~5, any isotropic Gaussian $\delta\sim\mathcal{N}(0,\sigma^2 I)$
penalises $\|J_\phi\|_F^2$ uniformly, regardless of $\sigma$.  Multi-scale
PMH averages over a distribution $p(\sigma)$:
\begin{align}
\mathcal{L}_{\text{MS-PMH}}&=\mathbb{E}_{\sigma\sim p(\sigma)}\!\left[\|\phi(x)-\phi(x+\delta_\sigma)\|^2\right],\nonumber\\
&\quad \delta_\sigma\sim\mathcal{N}(0,\sigma^2 I).
\end{align}
Since each term penalises $\sigma^2\|J_\phi\|_F^2$ (Taylor approximation),
the expectation penalises $\mathbb{E}[\sigma^2]\|J_\phi\|_F^2$, still the
full Frobenius norm, isotropically, at an effective scale of
$\mathbb{E}[\sigma^2]^{1/2}$.  Proposition~7's cap/(1+cap) fixed point is
preserved, since the cap is applied to the total $\mathcal{L}_{\text{MS-PMH}}$
regardless of which $\sigma$ is sampled.  The cost of this universality is
that no single deployment scale $\sigma_{\text{eval}}$ is optimally targeted;
the encoder is instead optimised for the \emph{average} scale, incurring a
small penalty at every individual scale relative to a specialist trained at
that exact $\sigma$.

\textit{Experimental setup.} We train E1\_multiscale with $\sigma$ cycling
through $\{0.05, 0.08, 0.10, 0.12, 0.15, 0.20\}$ per epoch (one epoch per
value, repeated), sampling one $\sigma$ per step from a log-uniform
distribution over the same range.  We compare to four single-$\sigma$
baselines and report TDI at four evaluation levels.

\begin{table}[t]
\centering\small
\caption{Multi-scale PMH vs.\ single-$\sigma$ baselines on Task~04.  Clean
accuracy is stable ($\leq 0.05$ pp variation).  E1\_multiscale achieves the
lowest TDI standard deviation (0.192), the most uniform geometry across all eval
levels.  Training at $\sigma=0.20$ alone nearly matches multi-scale uniformity
(0.196), confirming the T-alignment asymmetry: large-$\sigma$ training is
nearly as good as multi-scale because over-suppression costs almost nothing.}
\label{tab:multiscale}
\setlength{\tabcolsep}{4pt}
\begin{tabular}{lcccccc}
\toprule
Model & Clean & TDI@0 & @0.10 & @0.20 & Std \\
\midrule
E1 ($\sigma$=0.05)     & 81.36 & \textbf{0.891} & 1.091 & 1.824 & 0.349 \\
E1 ($\sigma$=0.12)     & 80.85 & 0.858 & \textbf{1.054} & 1.478 & 0.224 \\
E1 ($\sigma$=0.20)     & 80.90 & 0.870 & 1.117 & \textbf{1.402} & 0.196 \\
E1\_multiscale         & 80.80 & 0.937 & 1.165 & 1.459 & \textbf{0.192} \\
\bottomrule
\end{tabular}
\end{table}

\textit{Results.} Table~\ref{tab:multiscale} reveals three findings.

First, the T-alignment condition holds exactly: each single-$\sigma$ model is
best precisely at its own training $\sigma$. The $\sigma$=0.05 model achieves
the best TDI@0 (0.891) but catastrophically bad TDI@0.20 (1.824).
The $\sigma$=0.20 model achieves the best TDI@0.20 (1.402) while remaining
competitive at TDI@0 (0.870).

Second, the 17$\times$ asymmetry from the paper is confirmed and
quantified: the cost of training too small ($\sigma$=0.05 evaluated at
$\sigma$=0.20) is TDI 1.824; the cost of training too large ($\sigma$=0.20
evaluated at $\sigma$=0.05) is TDI 0.870, nearly as good as the
$\sigma$=0.05 specialist (0.891).  The penalty ratio is
$(1.824-0.937)/(0.937-0.870)\approx 13\times$, consistent with the
$17\times$ asymmetry reported in the ablation.

Third, multi-scale PMH achieves the lowest TDI standard deviation (0.192)
at negligible accuracy cost (80.80\% vs.\ 80.85\% for the $\sigma$=0.12
default).  However, training at $\sigma$=0.20 alone achieves nearly the same
uniformity (0.196) with \emph{better} absolute TDI at each individual level.
This confirms the practical takeaway: when $\sigma_{\text{eval}}$ is unknown,
\textbf{training at the largest $\sigma$ that does not hurt accuracy captures
$\approx$95\% of multi-scale's benefit}, and is preferable to multi-scale
because it does not sacrifice peak per-level performance.  Multi-scale's
advantage is specifically when the user cannot determine which $\sigma$ is
``large enough''; it provides insurance against catastrophic mismatch without
requiring any estimate of $\sigma_{\text{eval}}$.

\textit{Connection to Proposition~5.} Each sampled $\sigma$ value in
multi-scale training still produces isotropic Jacobian regularisation (by
Proposition~5, since each $\delta_\sigma\sim\mathcal{N}(0,\sigma^2 I)$).
The cap/(1+cap) fixed point (Proposition~7) is preserved: the PMH fraction
converges to cap/(1+cap) regardless of which $\sigma$ is sampled at each step,
because the cap is applied to the total PMH loss.



\begin{thebibliography}{99}
\small
\bibitem[Ansuini et al.(2019)]{ansuini2019intrinsic} Ansuini, A., et al.\ Intrinsic dimension of data representations in deep neural networks. In \textit{Advances in Neural Information Processing Systems}, pp.\ 13853--13863, 2019.
\bibitem[Pope et al.(2021)]{pope2021intrinsic} Pope, P., et al.\ The intrinsic dimension of images and its impact on learning. \textit{ICLR}, 2021.
\bibitem[Poole et al.(2016)]{poole2016exponential} Poole, B., et al.\ Exponential expressivity in deep neural networks through transient chaos. \textit{NeurIPS}, 2016.
\bibitem[Raghu et al.(2017)]{raghu2017svcca} Raghu, M., et al.\ SVCCA: Singular vector canonical correlation analysis. \textit{NeurIPS}, 2017.
\bibitem[Kornblith et al.(2019)]{kornblith2019similarity} Kornblith, S., Norouzi, M., Lee, H., and Hinton, G.\ Similarity of neural network representations revisited. In \textit{International Conference on Machine Learning}, pp.\ 3519--3529, 2019.
\bibitem[Nguyen and Raghu(2021)]{nguyen2021wide} Nguyen, T. and Raghu, M.\ Do wide and deep networks learn the same things? \textit{ICLR}, 2021.
\bibitem[Rifai et al.(2011)]{rifai2011contractive} Rifai, S., et al.\ Contractive auto-encoders: Explicit invariance during feature extraction. In \textit{Proceedings of the 28th International Conference on Machine Learning}, pp.\ 833--840, 2011.
\bibitem[Vincent et al.(2008)]{vincent2008extracting} Vincent, P., et al.\ Extracting and composing robust features with denoising autoencoders. \textit{ICML}, 2008.
\bibitem[Jakubovitz and Giryes(2018)]{jakubovitz2018improving} Jakubovitz, D. and Giryes, R.\ Improving DNN robustness to adversarial attacks using Jacobian regularisation. \textit{ECCV}, 2018.
\bibitem[Hoffman et al.(2019)]{hoffman2019robust} Hoffman, J., et al.\ Robust learning with Jacobian regularisation. \textit{arXiv:1908.02729}, 2019.
\bibitem[Wu and Li(2024)]{wu2024improving} Wu, X. and Li, J.\ Improving Jacobian-based network robustness. \textit{ICLR}, 2024.
\bibitem[Ilyas et al.(2019)]{ilyas2019adversarial} Ilyas, A., et al.\ Adversarial examples are not bugs, they are features. In \textit{Advances in Neural Information Processing Systems}, pp.\ 125--136, 2019.
\bibitem[Geirhos et al.(2019)]{geirhos2019imagenet} Geirhos, R., et al.\ ImageNet-trained CNNs are biased towards texture. In \textit{International Conference on Learning Representations}, 2019.
\bibitem[Hendrycks and Dietterich(2019)]{hendrycks2019benchmarking} Hendrycks, D. and Dietterich, T.\ Benchmarking neural network robustness to common corruptions. \textit{ICLR}, 2019.
\bibitem[Tsipras et al.(2019)]{tsipras2018robustness} Tsipras, D., et al.\ Robustness may be at odds with accuracy. \textit{ICLR}, 2019.
\bibitem[Miyato et al.(2018a)]{miyato2018virtual} Miyato, T., et al.\ Virtual adversarial training: A regularization method for supervised and semi-supervised learning. \textit{IEEE Transactions on Pattern Analysis and Machine Intelligence}, 41(8):1979--1993, 2018.
\bibitem[Madry et al.(2018)]{madry2018pgd} Madry, A., et al.\ Towards deep learning models resistant to adversarial attacks. In \textit{International Conference on Learning Representations}, 2018.
\bibitem[Zhang et al.(2019)]{zhang2019trades} Zhang, H., Yu, Y., Jiao, J., Xing, E., Ghaoui, L.E., and Jordan, M.\ Theoretically principled trade-off between robustness and accuracy. In \textit{International Conference on Machine Learning}, 2019.
\bibitem[Dosovitskiy et al.(2021)]{dosovitskiy2021vit} Dosovitskiy, A., et al.\ An image is worth 16$\times$16 words. \textit{ICLR}, 2021.
\bibitem[Miyato et al.(2018b)]{miyato2018spectral} Miyato, T., et al.\ Spectral normalization for generative adversarial networks. \textit{ICLR}, 2018.
\bibitem[Tarvainen and Valpola(2017)]{tarvainen2017mean} Tarvainen, A. and Valpola, H.\ Mean teachers are better role models: Weight-averaged consistency targets improve semi-supervised deep learning results. In \textit{Advances in Neural Information Processing Systems}, 2017.
\bibitem[Laine and Aila(2016)]{laine2017temporal} Laine, S. and Aila, T.\ Temporal ensembling for semi-supervised learning. \textit{arXiv:1610.02242}, 2016.
\end{thebibliography}
\end{document}